\documentclass{article}

\usepackage[english]{babel}

\usepackage[letterpaper,top=2cm,bottom=2cm,left=3cm,right=3cm,marginparwidth=1.75cm]{geometry}

\usepackage[utf8]{inputenc} % allow utf-8 input
\usepackage[T1]{fontenc}    % use 8-bit T1 fonts
\usepackage{hyperref}       % hyperlinks
\usepackage{url}            % simple URL typesetting
\usepackage{booktabs}       % professional-quality tables
\usepackage{amsfonts}       % blackboard math symbols
\usepackage{nicefrac}       % compact symbols for 1/2, etc.
\usepackage{microtype}      % microtypography
\usepackage[table]{xcolor}  % colors; [table] enables \rowcolor for tables

\usepackage{graphicx}
\usepackage{amssymb}
\usepackage{amsthm}
\usepackage{amsmath}
\newtheorem{theorem}{Theorem}
\newtheorem{proposition}{Proposition}
\newtheorem{example}{Example}

\newtheorem{definition}{Definition}
\newtheorem{corollary}{Corollary}
\usepackage{booktabs}
\usepackage{cleveref}
\usepackage{natbib}
\usepackage{algorithm}
\usepackage{algorithmic}
\usepackage{color}
\usepackage{enumitem}

\usepackage{xcolor}
\usepackage[most]{tcolorbox}
\usepackage{wrapfig}
\usepackage{subcaption}
\newtcolorbox{defbox}{
  colback=green!5!white,
  colframe=green!50!black,
  fonttitle=\bfseries,
  title=Definition
}

\usepackage[toc,page,header]{appendix}
\usepackage{minitoc}

\date{}
\title{CurveRL: Principled Distribution-Aware Context Reweighting for LLM Reasoning}
\author{%
  Ke Sun\footnote{Equal contribution. \quad $^\dagger$Co-corresponding authors.} $^{1}$ \quad
  Yizhou Zhao$^{*1}$ \quad
  Jiayi Xin$^{1}$ \quad
  Qi Long$^{\dagger 1}$ \quad
  Weijie Su$^{\dagger 1}$ \\
  $^{1}$University of Pennsylvania \\
  \texttt{\{kesun6, qlong\}@upenn.edu, yzzhao@sas.upenn.edu} \\
  \texttt{jiayixin@seas.upenn.edu, suw@wharton.upenn.edu}
}

\begin{document}
\maketitle

\begin{abstract}
 Context or prompt-level reweighting has emerged as a central algorithmic lever in Reinforcement Learning with Verified Rewards (RLVR) for improving the reasoning capability of large language models, yet the principle determining what constitutes an optimal weighting remains poorly understood. We address this gap by formulating prompt reweighting as a functional derivative of a utility functional defined in the pass-rate function space, yielding a unified optimality framework that accommodates existing schemes, including REINFORCE and GRPO. Building on this optimality framework, we propose a distribution-aware prompt reweighting approach, called \textbf{CurveRL}, based on a quantile coordinate transform, in which the weight assigned to each prompt depends not on the absolute value of pass rates but on its rank and density to reflect the distributional structure of the pass rates in the learning dynamics. Extensive experiments across multiple benchmarks demonstrate that our proposed CurveRL consistently outperforms GRPO and other RLVR baselines. Our study identifies context-distribution control as a principled axis for analyzing and designing prompt-reweighted RLVR algorithms. The code is released in \url{https://github.com/zhyzmath/CurveRL}. 
\end{abstract}

% \tableofcontents

%%%%%%%%%%% set table of content for appendix (start)
% Make the "Part I" text invisible
\renewcommand \thepart{}
\renewcommand \partname{}
\doparttoc % Tell to minitoc to generate a toc for the parts
\faketableofcontents % Run a fake tableofcontents command for the partocs
%%%%%%%%%%% set table of content for appendix (end)

\section{Introduction}

Reinforcement learning with verifiable rewards~(RLVR) has been the primary driver behind the recent emergence of reasoning models~\citep{jaech2024openai,guo2025deepseek}. As outcome-based rewards become more adopted, the token-level MDP~\citep{puterman2014markov} effectively collapses to a contextual bandit~\citep{lattimore2020bandit}, in which the entire reasoning trace is absorbed into the response as a one-step decision. Among modern RLVR methods, Group Relative Policy Optimization~(GRPO)~\citep{shao2024deepseekmath,guo2025deepseek} and its variants, such as \citep{yu2025dapo,liu2025understanding,chu2025gpg,zhang2025gvpo,xiong2025reinforce,tajwar2026maximum}, 
have emerged as the dominant family, offering competitive performance with low memory overhead. The mechanisms underlying their success, however, remain poorly understood.

Understanding why these algorithms work requires recognizing a distinctive feature of RLVR that has no direct counterpart in standard RL: the ability to directly shape the context or prompt distribution from which training samples are drawn.  In standard RL with an external environment, the state-visitation distribution is mainly shaped indirectly through the agent's actions and exploration strategy~\citep{thrun1992efficient,ladosz2022exploration}, as the state or context distribution is often treated as exogenous and cannot be directly manipulated. The contextual-bandit structure of outcome-based RLVR removes this constraint, opening up a new and orthogonal axis of information acquisition, which we refer to as \textit{context distribution control}: the prompt distribution is explicitly controllable during training, and the algorithm can decide which prompts to sample and how heavily to weight their gradients. Recently, a growing body of work exploits this freedom through a wide range of mechanisms, such as sample selection~\citep{yu2025dapo,mao2026dynamics,xiong2025reinforce}, curriculum strategies~\citep{parashar2025curriculum,rajaraman2026learning,chen2025self}, and prompt reweighting~\citep{davis2025objective,tajwar2026maximum}. These approaches are generally grounded by separate heuristics about which prompts deserve more gradient signal at each training stage. Yet none provides a principled answer to why its particular intervention on the prompt training distribution is the right one. Recent work~\citep{tajwar2026maximum} introduces the maximum likelihood principle into the RLVR objective by maximizing the log-likelihood of pass rates. While maximum likelihood estimation~(MLE) enjoys well-established optimality properties in classical statistics~\citep{shao1999mathematical,casella2024statistical}, we argue that these guarantees do not transfer to RLVR, as policy optimization is structurally different from statistical estimation. As also briefly mentioned in \citep{davis2025objective}, classical MLE optimality rests on a fixed, exogenous probability measure that characterizes the data-generating population, against which the estimator is evaluated. In RLVR, by contrast, the objective is evaluated under a policy-dependent measure or data distribution that co-evolves with the policy throughout training. There is no fixed population to estimate, and the measure itself is a policy-dependent endogenous object of the optimization. Thus, the classical optimality argument in MLE no longer applies. 

% A detailed discussion of related work is given in Appendix~\ref{app:relatedwork}.

\paragraph{Motivation: Prompt Reweighting and Its Optimality.} Prompt reweighting offers a concrete way to implement context distribution control in RLVR. To learn the LLM policy $\pi_\theta$, we consider RLVR algorithms whose policy gradient update assigns a policy-dependent prompt weight $w_\theta(x)$ to each prompt $x$. Let $r(x, y)$ denote the rule-based binary reward function for prompt $x$ and response $y$, $d_0$ denote the initial prompt distribution. Define the \textbf{pass rate} $p_\theta(x)=\mathbb{E}_{y \sim \pi_\theta(\cdot \mid x)}\left[r(x, y)\right]$. A broad class of prompt-reweighted RLVR algorithms employ the following policy gradient update:
\begin{align}\label{eq:gradentrule_generic}
     \nabla_\theta J(\theta)= \mathbb{E}_{x \sim d_0} \left[w_\theta(x) \mathbb{E}_{y \sim \pi_\theta(\cdot | x)} \left[ r(x, y) \nabla_\theta 
    \log \pi_\theta(y|x)\right] \right]= \mathbb{E}_{x \sim d_0}\left[w_\theta(x) \nabla_\theta p_\theta(x)\right].
\end{align}
For instance, the population counterpart of GRPO corresponds to $w_\theta(x)=1/\sqrt{p_\theta(x) (1-p_\theta(x))}$. A detailed explanation is given in Section~\ref{sec:preliminaries}. Importantly, this generic form in Eq.~\eqref{eq:gradentrule_generic} raises the central question of our study:
\begin{center}
    \textit{What are the principle and the optimality to determine $w_\theta(x)$ in the prompt-reweighted RLVR?}
\end{center}
% Answering the question above can also deepen our theoretical understanding about why GRPO employs a normalization-type prompt reweighting function and thus promotes the design of principled algorithms in the future.

\paragraph{Our Contributions.} In this paper, we cast prompt reweighting as context distribution control, where the algorithm directly reshapes the effective prompt distribution. Under this view, we formulate the optimal weight as a functional derivative of a utility function over the pass-rate function space.
The resulting optimality framework subsumes existing pointwise weighting rules and reveals their \textit{weight collapse} limitation. We then instantiate the principle with CurveRL, with the optimal weight derived by a distribution-aware utility function in pass-rate quantile space, which uses rank and density information of the evolving pass-rate distribution. Empirically, CurveRL consistently improves the pass@$1$ and pass@$k$ trade-off across multiple reasoning benchmarks. The contributions of our study can be succinctly summarized as follows:
\begin{itemize}[leftmargin=*]
    \item We formulate prompt reweighting in RLVR as context distribution control and define optimal weights through utility-dependent functional derivatives.

    \item We instantiate this principle with a distribution-aware utility in pass-rate quantile space and propose CurveRL that characterizes the distributional structure of the pass-rate distribution.
    
    \item We perform extensive experiments, showing that CurveRL improves the pass@$1$ and pass@$k$ Pareto frontier over standard baselines. The underlying mechanism of CurveRL is also analyzed.
\end{itemize}

\section{Preliminaries and Technical Background}\label{sec:preliminaries}

\paragraph{REINFORCE Objective in the Pass-Rate Space.} RLVR is often formulated as a contextual bandit problem, where the entire reasoning trace is absorbed into the response as a one-step decision. We assume a rule-based binary reward $r(x, y) \in \{0, 1\}$ for prompt $x$ and response $y$. To learn $\pi_\theta$ guided by these rewards, the policy gradient method~\citep{sutton1998reinforcement}, e.g., REINFORCE~\citep{williams1992simple,sutton1999policy}, is generally utilized to maximize $J_{\text{RL}}(\theta)$, which is defined by the following:
\begin{align}\label{eq:RL}
     J_{\text{RL}}(\theta)=\mathbb{E}_{x \sim d_0,y \sim \pi_\theta(\cdot \mid x)}\left[r(x, y)\right] =  \mathbb{E}_{x \sim d_0} \left[\mathbb{E}_{y \sim \pi_\theta(\cdot \mid x)}\left[\mathbb{I}\{y \in C(x)\}\right]\right]:= \mathbb{E}_{x \sim d_0} \left[p_\theta(x)\right],
\end{align}
where $C(x)$ is a feasible set in the response, which is often determined by the domain-specific verifier. The score function trick is pivotal in derivations of REINFORCE~\citep{williams1992simple}, offering a generic optimization tool to solve the optimization problem under decision-dependent sampling distributions for LLM-based generative model $\pi_\theta$. Next, we can derive its gradient:
\begin{align}
   \nabla_\theta J_{\text{RL}}(\theta)
    = \mathbb{E}_{x \sim d_0, y \sim \pi_\theta(\cdot | x)} \left[ r(x, y) \nabla_\theta 
    \log \pi_\theta(y|x)\right] =   \mathbb{E}_{x \sim d_0} \left[ \nabla_\theta p_\theta(x)\right].
\end{align}
In particular, for the generic prompt reweighting form in Eq.~\eqref{eq:gradentrule_generic} in the pass-rate space, REINFORCE corresponds to a constant prompt weight $w_\theta(x)=1$.

\paragraph{Group Relative Policy Optimization~(GRPO)~\citep{shao2024deepseekmath,guo2025deepseek}.} We next rewrite the gradient of GRPO in the same prompt-reweighting form. For each prompt $x$, GRPO performs $n$ rollouts to generate responses $\{y_i\}_{i=1}^n$. We denote a reference policy $\pi_{\text{ref}}$ and an old policy $\pi_{\text{old}}$. Under a contextual bandit or one-step decision abstraction analysis adopted in \citep{davis2025objective}, GRPO can be written in a sequence-level form, which is closely related to Group Sequence Policy Optimization~(GSPO)~\citep{zheng2025group}. As such, GRPO maximizes the following objective:
\begin{equation}
    \begin{aligned}
         &J_{\mathrm{GRPO}}(\theta) = \mathbb{E}_{x \sim d_0, \{y_i\}_{i=1}^n \sim \pi_{\theta_{\text{old}}}(\cdot \mid x)}\\
    & \left[\frac{1}{n} \sum_{i=1}^n\left(\min \left(\frac{\pi_{\theta}\left(y_i |x\right)}{\pi_{\theta_{\mathrm{old }}}\left(y_i | x\right)} \hat{A}_i^x, \operatorname{clip}\left(\frac{\pi_{\theta}\left(y_i | x \right)}{\pi_{\theta_{\text{old}}}\left(y_i | x \right)}, 1-\epsilon, 1+\epsilon\right) \hat{A}_i^x\right)-\left.\beta D_{\mathrm{KL}}\left(\pi_\theta(\cdot \mid x) \| \pi_{\mathrm{ref}}(\cdot \mid x)\right)\right)\right)\right],
    \end{aligned}
\end{equation}
where $\epsilon$ is the clipping parameter and $\hat{A}_i^x$ is the advantage estimate. The clipped policy ratio $\frac{\pi_{\theta}\left(y_i |x\right)}{\pi_{\theta_{\mathrm{old }}}\left(y_i |x\right)}$ prevents $\pi_\theta$ from deviating dramatically from the previous policy $\pi_{\text{old}}$ and the regularization hyperparameter $\beta$ penalizes the divergence from the reference policy $\pi_{\text{ref}}$. Notably, the group-based normalization scheme is employed to estimate the advantage by $\hat{A}_i^x$ in the gradient update rule:
\begin{align}
    \hat{A}_i^x=\frac{r\left(x, y_i\right)-\text{mean}\left(\left\{r\left(x, y_j\right)\right\}_{j=1}^n\right)}{\text{std}\left(\left\{r\left(x, y_j\right)\right\}_{j=1}^n\right)}.
\end{align}
Following \citep{davis2025objective}, when we replace the group-wise empirical baseline and normalization by their population counterparts~(i.e., an infinite group size) and ignore the clipping and the policy-ratio term, the gradient of GRPO can be approximated by the following simple form:
\begin{equation}
    \begin{aligned}\label{eq:GRPO}
    \nabla_\theta J_{\mathrm{GRPO}}(\theta) 
    &\approx \mathbb{E}_{x \sim d_0, y \sim \pi_\theta(\cdot \mid x)}\left[\frac{ r(x, y) - \mathbb{E}_{\tilde{y}\sim \pi_\theta(\cdot|x)}\left[r(x, \tilde{y})\right]}{\sqrt{\text{Var}_{\tilde{y} \sim \pi_\theta(\cdot \mid x)}\left[r(x, \tilde{y})\right]}}\nabla_\theta \log \pi_\theta(y|x)\right] \\
    & = \mathbb{E}_{x \sim d_0}\left[\frac{1}{\sqrt{p_\theta(x) (1-p_\theta(x))}} \nabla_\theta p_\theta(x) \right],
\end{aligned}
\end{equation}
which ideally corresponds to a population-level objective $J_{\text{GRPO}}(\theta) = \mathbb{E}_{x\sim d_0}\left[2 \arcsin{\sqrt{p_\theta( x)}}\right]$~\citep{davis2025objective}. Notably, the prompt weighting function of the population-level GRPO in Eq.~\eqref{eq:GRPO} is $1/\sqrt{p_\theta(x)(1-p_\theta(x))}$, which emphasizes prompts with either very small~(i.e., $p_\theta(x)\rightarrow 0$) or large~(i.e., $p_\theta(x)\rightarrow 1$) pass rates in the learning dynamics.

\paragraph{Advanced Prompt Reweighting RL Objectives.} Several recent RLVR variants can also be interpreted as pointwise transformations
of the pass rate~\citep{walder2025pass,tajwar2026maximum,xiong2025reinforce}. A representative example is Maximum Likelihood RL~(MaxRL)~\citep{tajwar2026maximum}, where the Maximum Likelihood~(ML) principle~\citep{bishop2006pattern,casella2024statistical} is heuristically introduced into the RL objective:
\begin{align}\label{eq:maxRL}
    J_{\text{ML}}(\theta)=\mathbb{E}_{x \sim d_0} \left[\log p_\theta( x)\right], \quad \nabla_\theta J_{\text{ML}} & = \mathbb{E}_{x\sim d_0}\left[\frac{1}{p_\theta(x)} \nabla_\theta p_\theta(x)\right].
\end{align}
An interesting property of MaxRL is that maximum likelihood optimizes an infinite harmonic mixture of pass@k gradients and REINFORCE therefore can be viewed as optimizing a first-order approximation to MaxRL. In practice, MaxRL approximates the maximum likelihood gradient in Eq.~\eqref{eq:maxRL} by truncating the Maclaurin expansion of the logarithmic function to a finite order. In our study, we use the oracle maximum likelihood form for theoretical analysis and its practical approximation for algorithmic comparisons.  In parallel, \cite{xiong2025reinforce} studied a prompt reweighting framework that adaptively allocates the number of sampled responses, implicitly prioritizing difficult prompts to mitigate vanishing pass-rate estimates. In contrast, we focus on identifying what constitutes an optimal weighting scheme itself and characterizing its role in shaping the learning dynamics, rather than approaching the problem from the perspective of sampling budget allocation.

In summary, the RLVR algorithms considered in this section can be interpreted as assigning different pointwise weights to prompts over the pass-rate function space. This perspective offers a technical foundation for the utility-functional framework developed in the remainder of the paper, particularly in \Cref{sec:optimality}.

\section{Prompt Reweighting as Utility-Dependent Context Distribution Control}\label{sec:optimality}

\subsection{Prompt Reweighting as Context Distribution Control}\label{sec:optimality_framework}

\paragraph{Problem Setting: Policy-Reweighted Contextual Bandit~(PRCB).} Introducing prompt reweighting into the RLVR policy-gradient update makes the \textit{effective context distribution} policy-dependent in the learning dynamics. Given a generic non-negative weight $w_\theta(x)$, define $d_\theta(x) = d_0(x) w_\theta(x) / Z_\theta$, where $Z_\theta =\int d_0(x) w_\theta(x) d x$ is the normalization factor. Under this framework, the general gradient update can be written as
\begin{align}\label{eq:gradient_reweighted}
    \nabla_\theta J \left(\theta \right)=\mathbb{E}_{x \sim d_0}\left[w_\theta(x) \nabla_\theta p_\theta(x)\right] = Z_\theta \mathbb{E}_{x \sim d_\theta}\left[\nabla_\theta p_\theta(x)\right] ,
\end{align}
where the positive scalar $Z_\theta$ only changes the update scale rather than the gradient update direction at a fixed policy iteration. Thus, prompt reweighting can be interpreted as optimizing under an effective prompt measure $d_\theta$ rather than the fixed base measure $d_0$. A special case is the \textbf{pointwise objective} $J_g(\theta)$ with its gradient form:
\begin{align}\label{eq:objective_pointwise}
    J_g(\theta) = \mathbb{E}_{x \sim d_0}\left[g\left(p_\theta(x)\right)\right], \quad \nabla_\theta J_g(\theta)=\mathbb{E}_{x \sim d_0}\left[g^{\prime}\left(p_\theta(x)\right) \nabla_\theta p_\theta(x)\right],
\end{align}
which induces the pointwise weight $w_\theta(x) = g^\prime(p_\theta(x))$ for a deterministic function $g$, and includes REINFORCE, GRPO, and MaxRL as examples from Section~\ref{sec:preliminaries}.  We call this setting a \textit{Policy-Reweighted Contextual Bandit~(PRCB)}: the contextual-bandit interaction remains one-step and outcome-based, but the optimization measure over contexts becomes policy-dependent through reweighting. Unlike MDP and non-Markovian environments that are treated as the temporal structure extension of classical contextual bandit, PRCB highlights a distinctive axis of information acquisition in RLVR in the direction of measure adaptivity via direct context distribution control. Notably, a direct context distribution control is typically unavailable in standard RL with an external environment. 

% We provide a detailed discussion and list new challenges in the PRCB setting in Appendix~\ref{app:discussion_PRCB}.

\paragraph{New Information Acquisition in RLVR.} In standard RL with an external environment, information acquisition is primarily achieved through action-level exploration~\citep{thrun1992efficient,ladosz2022exploration}, where the policy affects
future observations only indirectly through the environment dynamics. In contrast, RLVR under the PRCB formulation introduces a fundamentally new control axis: the learner can directly reshape the effective context distribution through policy-dependent prompt reweighting. This mechanism is closely related
in spirit to active learning~\citep{liu2024dual,menard2021fast}, adaptive experimental design~\citep{mehta2021experimental,blau2022optimizing}, and Bayesian optimization~\citep{frazier2018tutorial,brochu2010tutorial,balakrishnan2020efficient}, where learning efficiency is improved by adaptively allocating sampling effort
to more informative inputs. With a slight notational abuse, we let $d_t(x)$ denote the context distribution and $d_t(x,y)$ the joint context-response distribution at time $t$. The following schematic contrast highlights the distinction between the two paradigms:
% \begin{figure}[htbp]
% \centering
% \vspace{-5mm}
% \hspace{0mm}
% \begin{minipage}{0.06\textwidth}
%     \centering
%     \includegraphics[width=0.9\linewidth,trim=250 120 280 30, clip]{fig/Agent.png}
% \end{minipage}
% \hspace{2mm}
% \begin{minipage}{0.8\textwidth}
% \begin{align*}
% \text{Standard RL:} \ 
% & \cdots \rightarrow d_t(x) \overset{\pi_{\theta_t}}{\rightarrow} d_t(x,y) 
% \overset{\text{update}}{\rightarrow} \pi_{\theta_{t+1}} 
% \overset{\textbf{exploration}}{\rightarrow} d_{t+1}(x) \rightarrow \cdots, \\
% \text{RLVR under PRCB:} \  
% & \cdots \rightarrow d_t(x) \overset{\pi_{\theta_t}}{\rightarrow} d_t(x,y) 
% \overset{\text{update}}{\rightarrow} \pi_{\theta_{t+1}}  
% \overset{\boldsymbol{w_{\theta_{t+1}}(x)}}{\rightarrow} d_{t+1}(x) \rightarrow \cdots.
% \end{align*} 
% \end{minipage}
% \vspace{-3mm}
% \end{figure}
\begin{align*}
\text{Standard RL:} \ 
& \cdots \rightarrow d_t(x) \overset{\pi_{\theta_t}}{\rightarrow} d_t(x,y) 
\overset{\text{update}}{\rightarrow} \pi_{\theta_{t+1}} 
\overset{\textbf{exploration}}{\rightarrow} d_{t+1}(x) \rightarrow \cdots, \\
\text{RLVR under PRCB:} \  
& \cdots \rightarrow d_t(x) \overset{\pi_{\theta_t}}{\rightarrow} d_t(x,y) 
\overset{\text{update}}{\rightarrow} \pi_{\theta_{t+1}}  
\overset{\boldsymbol{w_{\theta_{t+1}}(x)}}{\rightarrow} d_{t+1}(x) \rightarrow \cdots.
\end{align*} 

This perspective highlights that RLVR differs fundamentally from standard RL by enabling direct control over the state or context distribution, rather than relying solely on action-level exploration.

\paragraph{New Challenges under PRCB.} The PRCB viewpoint also introduces two conceptual challenges. \textit{ (1) Coupled decision-dependent sampling structure}. RLVR already involves decision-dependent sampling in the response space, where the data distribution in the optimization depends on  $\pi_\theta$. Under PRCB, this dependence further extends to the context space through $d_\theta(x)$. As such, the learning dynamics are shaped jointly by response-level policy dependence via $\pi_\theta$ and context-level reweighting via $w_\theta$. \textit{(2) Breakdown of classical optimality principles.} Many optimality criteria for contextual bandits, such as simple and cumulative regret minimization~\citep{deshmukh2018simple,zhou2020neural,lattimore2020bandit}, are formulated under a fixed context distribution $d_0$. In PRCB, however, the effective training measure changes with the policy through $w_\theta$. This policy-dependent measure adaptivity falls outside the standard analysis and motivates our utility-based formulation as one principled optimality for the prompt reweighting introduced in
Section~\ref{sec:optimality_utility}.

\subsection{Optimality of Prompt Reweighting is Utility-Dependent}\label{sec:optimality_utility}

% and $p_\theta : \mathcal{X} \to [0,1]$ denote the pass rate function given $\pi_\theta$. 

% \begin{wrapfigure}[12]{r}{0.5\textwidth}
% % \begin{figure}[htbp]
%     \centering
%     \vspace{-10pt}
%     \includegraphics[width=1.0\linewidth,trim=0 0 0 0, clip]{fig/difficulty.png}
%     \vspace{-10pt}
%     \caption{Density shifts as a distribution transport.}
%     \label{fig:difficulty}
%     \vspace{-5pt}
% % \end{figure}
% \end{wrapfigure}

\paragraph{Policy Optimization in RLVR as Pass-rate Distribution Transport.}  Let $\mathcal{X}$ be the prompt space and $\pi^*$ be an oracle policy. Define an oracle pass-rate function $p^*: \mathcal{X}\rightarrow [0, 1]$ by 
$\mathbb{E}_{y \sim \pi^*(\cdot \mid x)} \left[r(x, y)\right]:= p^*(x)$. Let $f_\theta$ and $f^*$ be the corresponding density functions of the random variables $p_\theta(x)$ and $p^*(x)$ under $x \sim d_0$. The learning dynamics when performing policy optimization in RLVR can be interpreted as a distribution transport problem from $f_\theta$ to $f^*$ in the pass-rate space:
\begin{align}\label{eq:distribution_transport}
    f_\theta(t) \rightarrow f^*(t), \quad t \in [0, 1].
\end{align}
Importantly, perturbations to $p_\theta(x)$ at individual prompts induce heterogeneous effects on the resulting distribution transport, which are governed not only by local sensitivity about pass rates, but also by cross-prompt interactions, coupling, and interference in the induced distributional dynamics~\citep{he2019local,barakat2026pass}.

\paragraph{Optimal Weight Function as Marginal Contribution of a Utility Functional.} Assume $p_\theta \in L^2(d_0)$, the space of square-integrable functions under the prompt distribution $d_0$. We define a policy-dependent functional $\mathcal{U}_\theta: p_\theta(\cdot)\rightarrow \mathbb{R}$ on the pass-rate function space. Note that $\mathcal{U}_\theta$ integrates a broad transformation class of $p_\theta(\cdot)$ over $x \sim d_0$. Define a pushforward measure $\mu_\theta=\left(p_\theta\right)_{\#} d_0$~(equivalently, $\mu_\theta = \text{Law}_{x\sim d_0}(p_\theta(x))$), i.e., the distribution of $p_\theta(x)$ by mapping prompts $x$ to the pass rate function $p_\theta(\cdot)$, with the property that $\mu_\theta([0,1])=\mathbb{P}_{x \sim d_0}\left(p_\theta(x) \in[0,1]\right)=1$. In measure theory, $\mu_\theta$ induces the Cumulative Distribution Function~(CDF) $F_\theta$ defined by $F_\theta(t)=\mu_\theta([0, t])=\mathbb{P}_{x\sim d_0}\left(p_\theta(x) \leq t\right)$. We introduce the derivative concept to define the optimal prompt reweighting $w_\theta^\star$, which characterizes the fastest direction for increasing the marginal contribution of $\mathcal{U}_\theta$ through an upward shift of $p_\theta(x)$. Concretely, we define the optimal weight function $w^\star_\theta(x)$ as the functional derivative~\citep{lions1971optimal,evans2022partial} of the utility functional $\mathcal{U}_\theta$ over the pass rate function $p_\theta(\cdot)$ evaluated at the prompt $x$ in Definition~\ref{def:weight}.
\begin{definition}[\textbf{Utility-dependent Optimal Prompt Weight $w_\theta^\star$}]\label{def:weight} For any perturbation 
$h \in L^2(d_0)$, according to the Riesz representation, the first variation of $\mathcal{U}_\theta$ is given by
\begin{align}\label{eq:weight_Gateaux}
\lim_{\epsilon \to 0}
\frac{\mathcal U_\theta(p_\theta + \epsilon h) - \mathcal U_\theta(p_\theta)}{\epsilon}
=
\int w_\theta^\star(x)\, h(x)\, d_0(x) dx,
\end{align}
where $w_\theta^\star(x)$ is the functional derivative denoted by
\begin{align}
    w_\theta^\star(x)
:= \frac{\delta \mathcal{U}_\theta}{\delta p_\theta}(x),
\end{align}
which indicates the optimal prompt weighting function for the prompt $x$ and the current policy $\pi_\theta$.
\end{definition}

\noindent 
As the Gateaux derivative defines a linear functional 
in the perturbation direction $h$ (cf. Eq.~\eqref{eq:weight_Gateaux}),  the Riesz representation theorem~\citep{brezis2011functional} ensures that, under mild regularity conditions, this functional admits a representation in the Hilbert space $L^2(d_0)$. Particularly, there exists $w_\theta^\star \in L^2(d_0)$ such that the first variation of the objective can be expressed as an inner product with respect to $d_0$, which naturally defines the optimal weight function $w_\theta^\star(x)$. Our definition is reminiscent of the notion of influence functions widely studied in statistics~\citep{hampel1974influence,ronchetti2009robust} and recently revisited in modern machine learning~\citep{koh2017understanding,grosse2023studying,min2025understanding}, where one quantifies the impact of infinitesimal perturbations on a global objective. More broadly, the functional derivative generalizes the ordinary partial derivative from finite-dimensional variables to infinite-dimensional function spaces~\citep{lions1971optimal,evans2022partial}. Partial derivatives with respect to $p_\theta(x)$ only capture the local sensitivity of each $x$,  but \emph{functional derivatives can characterize how the objective responds to perturbations on the entire pass rate function} $p_\theta(\cdot)$. The functional derivatives-based definition generalizes the partial derivatives, enabling distribution-aware weighting that reflects the data geometry, akin to the geometry-aware formulations in optimal transport~\citep{villani2009optimal,peyre2019computational}. This perspective is also the theoretical foundation of our approach in Section~\ref{sec:distribution}. 

\begin{example}[Pointwise Utility Functional.] For the pointwise utility functional with $\mathcal{U}_\theta = J_g(\theta)$, the prompt weight only depends on $p_\theta(x)$ itself by implicitly assuming the independence among prompts. Therefore, the functional derivative degenerates to the partial derivative:
\begin{align}
    \frac{\delta \mathcal{U}_\theta}{\delta p_\theta}(x)=g^{\prime}\left(p_\theta(x)\right) = w^\star_\theta(x).
\end{align}
This reduction encompasses the aforementioned RLVR algorithms in a straightforward way, including
(1) $w^\star_\theta(x)=1$ in REINFORCE, (2) $w^\star_\theta(x) = 1/\sqrt{(p_\theta(x)(1-p_\theta(x))}$ in GRPO, and (3)$w^\star_\theta(x) = 1/p_\theta(x)$ in MaxRL, by taking the functional~(partial) derivative~(see Appendix~\ref{app:theory_reduction}).
\end{example}

\begin{example}[Variance-Based Utility Functional.] Functional derivatives over the pass-rate function are strictly more expressive than pointwise derivatives, i.e., in general $\frac{\delta \mathcal{U}_\theta}{\delta p_\theta}\left(x\right) \neq \frac{d \mathcal{U}_\theta(p_\theta(x))}{d p_\theta(x)}$.
Unlike pointwise utility functionals, functional derivatives can capture global distributional structure of $p_\theta(\cdot)$. As an illustrative example, consider the variance-based utility functional $\mathcal{U}_\theta=\operatorname{Var}_{x \sim d_0}(p_\theta(x))=\mathbb{E}_{x\sim d_0}\left[p_\theta(x)^2\right]-(\mathbb{E}_{x \sim d_0}[p_\theta(x)])^2$. We can derive
\begin{align}
    \frac{\delta \mathcal{U}_\theta}{\delta p_\theta}(x)=2 p_\theta(x)-2 \mathbb{E}_{x\sim d_0}[p_\theta(x)],
\end{align}
where the second term introduces \textit{a global coupling through the expectation}, reflecting the dependence on the entire pass-rate distribution rather than the local value $p_\theta(x)$ alone. This serves as a motivating example of the general distribution-aware utility functionals introduced in Section~\ref{sec:distribution}. 
\end{example}

\subsection{The Optimality of Utility Functionals is Risk-Dependent}\label{sec:optimality_risk}

\paragraph{Optimal Choice of the Utility Functional: A Risk-Sensitive Control Perspective.} According to Definition~\ref{def:weight}, the optimal prompt reweighting is determined by the choice of a utility function $\mathcal{U}_\theta$ that influences the performance of the RLVR algorithm. A natural question is whether there exists a single ``optimal'' utility function $\mathcal{U}_\theta$ that uniformly dominates the alternatives. However, we argue that the answer is negative, as it depends on the choice of risk preference. 
% Different risk preferences are optimal for different design objectives. 
The RL literature has long recognized that utility selection reflects the designer's risk attitude rather than a universal principle. Specifically, risk-sensitive RL~\citep{mihatsch2002risk} with an objective of some risk measure on cumulative rewards, has not converged on a single optimal risk measure. A typical strategy is to focus on concrete criteria such as exponential utility and CVaR, treating selection as a tradeoff among interpretability, safety semantics, and tractability~\citep{bauerle2024markov,wang2022risk,smith2023exponential,majumdar2017risk}. Distributional RL~\citep{dabney2018implicit,bellemare2023distributional} makes this dependence explicit by allowing arbitrary distortion functions. From this perspective, REINFORCE, GRPO, and MaxRL are different points in a continuum of risk preference. A useful summary of this continuum is the widely adopted entropic risk family~\citep{howard1972risk}, where the agent seeks to optimize the expectation of an exponential utility function $\mathcal{U}_\theta^{\mathrm{risk}}(\eta)$ with the lower bound $J_{\mathrm{RL}}$ by Jensen's inequality:
% which is also equivalent to a KL robust control, 
\begin{align}
   \mathcal{U}_\theta^{\mathrm{risk}}(\eta)=\mathbb{E}_{x \sim d_0}\left[ \frac{1}{\eta} \log  \mathbb{E}_{y \sim \pi_\theta(\cdot \mid x)} e^{\eta r(x, y)}\right] \geq \mathbb{E}_{x \sim d_0} \left[\mathbb{E}_{y \sim \pi_\theta(\cdot \mid x)}\left[r(x, y)\right]\right] = J_{\mathrm{RL}}(\theta),
\end{align}
where $\eta > 0$ controls the degree of risk-seeking: larger values of $\eta$ place exponentially more weight on high-reward
responses. As shown in Proposition~\ref{prop:entropic_risk}, this entropic-risk family provides a utility-functional interpolation  between standard RL~(i.e., REINFORCE) and MaxRL. Increasing $\eta$ induces stronger reweighting toward low-pass-rate prompts through the resulting weight function $w_\theta^\star$. The proof of Proposition~\ref{prop:entropic_risk} is provided in Appendix~\ref{app:theory_entropicrisk}. In summary, we argue that no single utility functional $\mathcal{U}_\theta$ is inherently preferred outside of a specific design objective. 

\begin{proposition}[\textbf{Gradient of Entropic Risk RL Interpolates Standard RL and MaxRL}]\label{prop:entropic_risk} Assume $p_\theta(x) > 0$ for $d_0$-almost every $x$ and $\frac{\left\|\nabla_\theta p_\theta(x)\right\|}{p_\theta(x)} \in L^1\left(d_0\right)$, then 
\begin{align}
      \lim _{\eta \rightarrow 0^+} \nabla_\theta \mathcal{U}_\theta^{\mathrm{risk}}(\eta) = \nabla_\theta J_{\mathrm{RL}}(\theta), \ \lim _{\eta \rightarrow \infty} \eta \nabla_\theta \mathcal{U}^{\mathrm{risk}}_\theta(\eta)=\nabla_\theta J_{\mathrm{ML}}(\theta).
\end{align}
\end{proposition}

% \clearpage
\section{CurveRL: Distribution-Aware Reweighting in Pass-Rate Quantile Space}\label{sec:distribution}

\subsection{From Pointwise Utility to Distribution-Aware Utility in the Quantile Space}\label{sec:distribution_method} 

Even though there may not be a universally optimal utility function in the general prompt reweighting as analyzed in \Cref{sec:optimality_risk}, we can still systematically improve the existing pointwise utility functional family in Eq.~\eqref{eq:objective_pointwise}.

\paragraph{Motivation: Weight Collapse of Pointwise Utility Functionals.} Previous approaches, such as \citep{davis2025objective,tajwar2026maximum,xiong2025reinforce}, define utility as a pointwise transformation of the pass rate, i.e., $g(p_\theta(x))$, solely based on the absolute magnitude $p_\theta(x)$. However, such strategies suffer from a systematic limitation that we term \textit{weight collapse}: the induced weights fail to adequately differentiate prompts with distinct learning potential. This issue manifests in both early and late training phases.
(1) In the early stage, most prompts have $p_\theta(x)\approx 0$, leading to nearly indistinguishable weights, despite substantial heterogeneity: some prompts lie near the learning frontier, while others are intrinsically difficult and less likely to benefit from further training. (2) In the late stage, many prompts have $p_\theta(x)\approx 1$ and thus yield similar weights, although certain prompts still admit meaningful improvement. The root cause is that pointwise utility functions depend solely on the local information in the pass rate space based on the \emph{absolute} value of $p_\theta(x)$, ignoring the \emph{geometry} in the pass rate space~(i.e., the \textit{distributional structure} of $p_\theta(x)$), including the rank, density, and spacing information. To address this limitation, we propose a fundamentally different distribution-aware utility function family. This motivation is analogous to the distinction between geometry-aware distances, such as Wasserstein distance, and pointwise divergences like Kullback–Leibler~(KL) divergence~(see Appendix~\ref{app:discussion_geometry} for more discussions).

\paragraph{Distribution-Aware Utility Functional.} We aim to incorporate the distributional structure of $p_\theta(x)$ into the prompt weighting $w_\theta(x)$. Consider a reference distribution~(i.e., probability measure) $\mu_{\mathrm{ref}}$, which induces a CDF $F_{\mathrm{ref}}$ with a density $f_{\mathrm{ref}}$ defined by $F_{\mathrm{ref}}(t)=\mu_{\mathrm{ref}}([0, t])$. We consider a \textit{quantile coordinate transform} via $F_{\text{ref}}$ on $p_\theta(x)$ to develop a new utility function:
\begin{align}\label{eq:utility_distribution}
    \mathcal{U}_\theta(F_{\text{ref}})=\mathbb{E}_{x \sim d_0}\left[\psi\left(F_{\text{ref}}\left(p_\theta(x)\right)\right)\right], 
\end{align}
where $\psi$ is an increasing function, often called \textit{distortion function} rooted in distortion risk measure in economics and risk theory~\citep{yaari1987dual,acerbi2002spectral,wang1996premium,dhaene2012remarks,balbas2009properties}. In principle, $F_{\text{ref}}$ should approximate the distribution of $p_\theta(x)$ to capture meaningful distributional geometry. Nonetheless, it should not coincide with the exact policy-induced CDF $F_{\theta}$ of the same random variable $p_\theta(x)$. This is because, by probability integral transform~\citep{casella2024statistical}, we have $F_{\theta}(p_\theta(x))\sim\text{Uniform}(0, 1)$, which removes any dependency on $\theta$ in expectation for $\mathcal{U}_\theta(F_\theta)$. Consequently, the utility degenerates to a constant: $\mathbb{E}_{x\sim d_0}\left[\psi\left(F_{\theta}\left(p_\theta(x)\right)\right)\right]=\int_0^1 \psi(z) d z$, yielding a non-informative $\mathcal{U}_\theta$. Proposition~\ref{prop:CDF_bound} further quantifies this effect by relating the deviation between $F_{\mathrm{ref}}$ and $F_\theta$ to a distribution-aware 1-Wasserstein distance. The proof is provided in Appendix~\ref{app:theory_CDFbound}.
\begin{proposition}\label{prop:CDF_bound}
    Denote $W_1$ as 1-Wasserstein distance. If $\psi$ is $L_\psi$-Lipschitz and $\Vert f_\theta \Vert_\infty < \infty$, then
    \begin{align}
        \left| \mathcal{U}_\theta(F_{\mathrm{ref}}) - \mathcal{U}_\theta(F_\theta) \right| \leq L_\psi \Vert f_\theta \Vert_\infty W_1(\mu_{\mathrm{ref}}, \mu_\theta).
    \end{align}
\end{proposition}

\paragraph{Implementation and Interpretation.} In practice, $F_{\text{ref}}$ is instantiated using a lagged policy~(e.g., $F_{\text{ref}} = F_{\theta_\text{old}}$) and updated in a sliding-window manner akin to the target network trick adopted in deep RL algorithms~\citep{mnih2015human,lillicrap2015continuous}. The gradient is taken with respect to the current pass rate $p_\theta(x)$ by fixing $F_{\text{ref}}$ within each gradient update, analogous to a standard pointwise derivative. Theoretically, $F_{\text{ref}}$ encodes the global information about the pass rate function $p_\theta(\cdot)$ and is thus policy-dependent. Consequently, although the gradient computation appears pointwise in implementation by fixing $F_{\text{ref}}$, the resulting utility functional $\mathcal{U}_\theta$ in Eq.~\eqref{eq:utility_distribution} is inherently distribution-aware and cannot be reduced to a deterministic pointwise transformation $g(p_\theta(\cdot))$. From the perspective of distribution transport illustrated in Section~\ref{sec:optimality_utility}, the introduced quantile-based representation via $F_{\mathrm{ref}}$ aligns with a Wasserstein-type geometry, rather than relying solely on the pointwise information like the KL divergence. Recall the definition of 1-Wasserstein distance as $W_1(\mu, \nu) =  \inf _{\gamma \in \Pi(\mu, \nu)} \int\|x-y\|_1 d \gamma(x, y) = \int_{\mathbb{R}}\left|F_\mu(t)-F_\nu(t)\right| d t$ for two measures $\mu$ and $\nu$, where $\Pi(\mu, \nu)$ is the joint coupling with the marginal distributions $\mu$ and $\nu$. Similarly, Wasserstein geometry is also formed on the quantile space via the CDFs $F_\mu$ and $F_\nu$. We defer a detailed discussion about the induced geometry in Appendix~\ref{app:discussion_geometry}.

\paragraph{Our Method: CurveRL.} In our main approach, we adopt $\psi(u) = \log u$ as the distortion risk function. This choice corresponds to a risk-seeking preference that emphasizes hard prompts, a preference well aligned with RLVR's empirical goal of improving worst-case reasoning performance. The log is preferable for designers who care about tail performance, but it is not universally optimal as discussed in Section~\ref{sec:optimality_risk}. Consequently, a concise and interpretable prompt reweighting is derived in the gradient form:
\begin{align}\label{eq:ours}
    \nabla_\theta \mathcal{U}_\theta(F_{\mathrm{ref}}) &  = \mathbb{E}_{x\sim d_0}\left[ \frac{f_{\mathrm{ref}}\left(p_\theta(x)\right)}{F_\mathrm{ref}\left(p_\theta(x)\right)} \nabla_\theta p_\theta(x)\right],
\end{align}
where the weight above is interpreted as emphasizing prompts with the \textit{lower-quantile} pass rates via $1/F_{\mathrm{ref}}(p_\theta(x))$ and \textit{higher density} via $f_{\mathrm{ref}}(p_\theta(x))$. The logarithm transformation as $\psi$ explicitly elicits the geometry of the pass rate function space of $p_\theta(\cdot)$, including the density and rank information of $p_\theta(x)$, to reflect the distributional structure. This weight has the same form as \textit{reverse hazard rate} in probability and survival analysis~\citep{block1998reversed,finkelstein2002reversed} albeit in a fundamentally different context.

\subsection{Algorithm: CurveRL}\label{sec:distribution_algorithm}

\begin{wrapfigure}{r}{0.48\linewidth}   
  \vspace{-8mm}                         
  \begin{minipage}{\linewidth}\small    
\begin{algorithm}[H]
\caption{CurveRL Update at Step $t$}
\label{alg:curverl}
\begin{algorithmic}[1]
\REQUIRE Batch $\mathcal{B}$ of inputs, number of rollouts $N$, sliding window $W$ with the size $t_0 \times |\mathcal{B}|$
\STATE {\color{gray} \# Weight Estimation }
\FOR{each $p \in [\frac{1}{N}, \frac{2}{N},\ldots, \frac{N-1}{N}]$}
    \STATE Evaluate density $\hat{f}_{\mathrm{ref}}(p)$ and CDF $\hat{F}_{\mathrm{ref}}(p)$
\ENDFOR
    \STATE {\color{gray} \# Gradient Estimation}
\FOR{each input $x \in \mathcal{B}$}
    \STATE Sample $y_{1},\dots,y_{N}\sim\pi_{\theta}(\cdot\mid x)$\\
    \FOR{$i = 1$ to $N$}
    \STATE $r_i \leftarrow r(x,y_{i})$, $S_i \leftarrow \nabla_\theta\log\pi_\theta(y_{i}\mid x)$
    % \STATE Evaluate 
    \ENDFOR
    \STATE Estimate $\hat p\leftarrow\tfrac{1}{N}\sum_{i=1}^{N} r_i $
    \IF{$\hat{p}\in (0,1)$}
        \STATE $w_t(\hat{p}) \leftarrow \hat{f}_{\mathrm{ref}}(\hat{p}) / \hat{F}_{\mathrm{ref}}(\hat{p})$
        \STATE $\hat{g}(x) \leftarrow \frac{1}{N} \sum_{i=1}^N w_t(\hat{p}) \left(r_i - \hat{p}\right)S_i$
        \STATE Append $\hat{p}$ to $W$
    \ENDIF
\ENDFOR
\STATE Remove pass rates in $t-t_0$ from $W$
\STATE $\hat{g} \leftarrow \frac{1}{|\mathcal{B}|} \sum_{x\in \mathcal{B}} \hat{g}(x)$
\RETURN $\hat{g}$
\end{algorithmic}
\end{algorithm}
  \end{minipage}
  \vspace{-15mm}
\end{wrapfigure}

\noindent In Algorithm~\ref{alg:curverl}, we elaborate on the update details of CurveRL at each training step $t$.

\paragraph{Sliding Window Estimate.} Denote the input batch $\mathcal{B}$ and number of rollouts $N$. We initialize a queue $W$ with the size $t_0 \times |\mathcal{B}|$, which stores the pass rates collected from the step $t-t_0$ to $t-1$. At step $t$, we estimate $F_{\mathrm{ref}}(p)$ and $f_{\mathrm{ref}}(p)$ via a histogram estimator from this lagged window for each $p \in [\frac{1}{N}, \frac{2}{N}, \ldots, \frac{N-1}{N}]$. Therefore, $F_{\mathrm{ref}}$ is determined by old policies in the last $t_0$ steps.

\paragraph{Weight and Gradient Estimation.} Next, for each prompt $x$ in the batch $\mathcal{B}$, we perform $N$ response attempts and estimate the pass rate $\hat{p}$. We evaluate $\hat{F}_{\mathrm{ref}}(\hat{p})$ and $\hat{f}_{\mathrm{ref}}(\hat{p})$ in the lagged window. We further use $\hat{p}$ as the group-level baseline, which reduces the variance of the policy gradient estimator. Finally, we maintain $W$ as an active-pass-rate window, appending only prompts with non-vanishing gradients, i.e., $\hat{p}\in (0,1)$. We also remove pass rates collected at step $t-t_0$, ensuring a length-$t_0$ sliding window estimation for $F_{\mathrm{ref}}$ in $W$. Thus, $t_0$ is the main hyperparameter in CurveRL.

\subsection{Theoretical Interpretation and Advantages}\label{sec:distribution_theoryadvantage}

\paragraph{A Unified Explanation of Prompt Reweighting as a Prior $F_{\mathrm{ref}}$.} The distribution-aware utility in Eq.~\eqref{eq:utility_distribution} provides a unified interpretation of pointwise prompt reweighted algorithms. Specializing to the log distortion $\psi(u) = \log u$ in Eq.~\eqref{eq:ours}, Theorem~\ref{thm:distribution} shows that any pointwise weight $w_\theta(x)$, e.g., those in REINFORCE, GRPO, and MaxRL, can be equivalently represented by a prior distribution $F_{\mathrm{ref}}$ in the gradient update. In contrast, CurveRL estimates $F_{\mathrm{ref}}$ from the evolving pass-rate distribution and is thus adaptive and data-driven. The proof of Theorem~\ref{thm:distribution} is given in Appendix~\ref{app:theory_distribution}.
\begin{theorem}[\textbf{Pointwise Weight Induces a Prior $F_{\mathrm{ref}}$}]\label{thm:distribution} Denote $p=p_\theta(x)$ and the pointwise weight $w_\theta(x)=g^\prime(p_\theta(x)):=w(p)$ in the pass-rate space. Assume $\int_p^1 w(t) dt < \infty$ for any $p \in (0, 1]$. Under the distortion $\psi(u)=\log u$ in Eq.~\eqref{eq:ours}, $w(p) = f_{\mathrm{ref}}\left(p\right)/F_{\mathrm{ref}}\left(p\right)$ admits a unique $F_\mathrm{ref}$:
\begin{align}\label{eq:prior}
    F_\mathrm{ref}(p) = \exp \left(-\int_p^1 w(t) dt \right).
\end{align}
% \vspace{-4mm}
\end{theorem}
Next, we consider the degeneration case of our distribution-aware prompt reweighting from the distribution of $F_{\mathrm{ref}}$. In \Cref{cor:degeneration}, we show that $F_{\mathrm{ref}}$ is $\text{Uniform}(0,1)$, i.e., $F_{\mathrm{ref}}(t)=t$, if and only if the gradients of distribution-aware and pointwise utility functionals coincide under the same distortion measure $\psi$, i.e., $\nabla_\theta \mathcal{U}_\theta(F_{\mathrm{ref}})=\nabla_\theta J_\psi(\theta)$. 
\begin{corollary}[Degeneration of Distribution-Aware Prompt Reweighting]\label{cor:degeneration}
    Consider the distribution-aware utility function $\mathcal{U}_\theta(F_{\mathrm{ref}})$ and the pointwise utility function $J_\psi(\theta)$ with the same distortion function $\psi$. Then,
    \begin{align}
        \nabla_\theta \mathcal{U}_\theta(F_{\mathrm{ref}})=\nabla_\theta J_\psi(\theta) \quad  \text{if and only if} \quad F_{\mathrm{ref}}(t)=t, \quad t\in [0,1].
    \end{align}
\end{corollary}
The proof of \Cref{cor:degeneration} is provided in Appendix~\ref{app:discussion_equivalence}. This identifies the uniform reference distribution as the boundary case in which the quantile transform does not change the gradient direction.  Away from this boundary, the effect of distribution-aware reweighting can be characterized by a relative multiplier $R_\psi(p) = \psi^\prime(F_{\mathrm{ref}}(p))f_{\mathrm{ref}}(p)/\psi^\prime(p)$. 

\paragraph{Sufficient Conditions for the Comparison of Weight Magnitude.} We further derive sufficient conditions to compare the weight magnitude between distribution-aware and pointwise prompt weighting methods based on $R_\psi(p)$. In particular, when $R^\prime_\psi(p)<0$, our distribution-aware method assigns relatively more normalized weights to low-pass-rate prompts as opposed to the pointwise counterpart without $F_{\mathrm{ref}}$, reflecting more risk-seeking behavior or aggressive reweighting against low-pass-rate prompts. The detailed discussion is provided in Appendix~\ref{app:discussion_weightcomparison}.

\paragraph{Advantage 1: Distribution-aware and Data-Driven Reweighting.} The classical weight function derived from the pointwise utility functional solely relies on the local information or raw value of $p_\theta(x)$ at each prompt. By contrast, our weight function is derived from a distribution-aware utility functional that characterizes the geometry~(e.g., ranking and density in Eq.~\eqref{eq:ours}) in the pass-rate function space based on the functional derivative defined in Definition~\ref{def:weight}. More importantly, as shown in Theorem~\ref{thm:distribution}, pointwise weights can be treated as special prior distribution of $F_{\mathrm{ref}}(p_\theta(x))$, while our weight function of CurveRL in Eq.~\eqref{eq:ours} is \textit{adaptive and data-driven}, which is updated periodically via non-parametric statistical approaches, e.g., a histogram function~(see Section~\ref{sec:distribution_algorithm}). Our method is more adaptive than the exogenous prompt schedule in curriculum learning~(see Appendix~\ref{app:discussion_curriculum}).

\paragraph{Advantage 2: Mitigating Weight Collapse by Introducing Quantile Coordinate.} As mentioned in Section~\ref{sec:distribution_method}, the  pointwise utility function suffers from the weight collapse issue when $p_\theta(x)\approx 0$ or $1$, leading to nearly indistinguishable learnability in the weights. By mapping the absolute value of pass rates to the CDF/quantile coordinate via $F_{\mathrm{ref}}$ in Eq.~\eqref{eq:utility_distribution}, CurveRL becomes robust to miscalibration of pass rates and enjoys a \textit{monotone calibration invariance property}~(see Appendix~\ref{app:theory_invariance}). Unlike the evolving non-stationary pass-rate space in the learning dynamics, the CDF/quantile space via $ F_\mathrm{ref}(p_\theta(x))$ in CurveRL induces a more uniform difficulty spectrum and is thus more stationary as $F_{\mathrm{ref}}$ is close to $F_\theta$ as analyzed in Proposition~\ref{prop:CDF_bound}. This potentially leads to more stable optimization.

\section{Experiments}
\label{sec:experiment}

\label{sec:exp-setup}
\paragraph{Experimental Setups.} We train Qwen3-1.7B-Base and Qwen3-4B-Base on POLARIS-53K~\citep{an2025polaris}, a corpus of approximately 53K mathematical reasoning prompts, using the verl framework~\citep{sheng2025hybridflow}. We compare CurveRL with GRPO~\citep{shao2024deepseekmath} and MaxRL~\citep{tajwar2026maximum}. All methods use a shared training loop and differ only in the prompt-weighting rule. We adopt the batch size $|\mathcal{B}|=256$, the number of rollouts $N=8$ for all algorithms, and $t_0=10$ in Algorithm~\ref{alg:curverl} for CurveRL. Training-time correctness is verified using \textsc{Math-Verify}~\citep{kydlicek_math_verify}. In evaluation, we use $\mathrm{pass@}k$~\citep{chen2021evaluating} as the primary metric, with $k\in\{1,2,4,\dots,1024\}$ for Qwen3-1.7B-Base and $k\in\{1,2,4,\dots,512\}$ for Qwen3-4B-Base. $\mathrm{pass@}k$ measures the probability that at least one of $k$ independently sampled responses is correct, serving as a proxy for a model's capability under a fixed sampling budget. In practice, $\mathrm{pass@}1$ is the raw mean
accuracy, while $\mathrm{pass@}k$ for $k\geq 2$ uses $1000$ best-of-$k$ bootstrap resamples (with replacement) of size $k$, averaged across prompts. We perform $2048$ rollouts per prompt for Qwen3-1.7B-Base and $1024$ for Qwen3-4B-Base to reduce estimation variance. More implementation details are provided in Appendix~\ref{app:impl-details}.

\paragraph{Evaluation.} We evaluate trained models on eight challenging mathematical reasoning
benchmarks:
\text{AIME 2025}~\citep{balunovic2025matharena}, \text{BeyondAIME}~\citep{seed2025thinking},
\text{HMMT 02/25}~\citep{balunovic2025matharena}, \text{HMMT 02/26}~\citep{balunovic2025matharena}, and
\text{MATH-500}~\citep{hendrycks2021measuring} in the main content, and
\text{BRUMO 2025}~\citep{balunovic2025matharena},
\text{HMMT 11/25}~\citep{balunovic2025matharena}, and
\text{Minerva}~\citep{lewkowycz2022solving} in Appendix~\ref{app:supp-results}.
% Together, these benchmarks span Olympiad-level problems (AIME, HMMT, BRUMO, BeyondAIME) and broader undergraduate-level mathematics (MATH-500, Minerva), allowing us to assess both peak reasoning ability and breadth of competence.

\subsection{Main Results}
\label{sec:main-results}

\begin{table}[b!]
\centering
% \vspace{-5mm}
\caption{\textbf{Main results on five math reasoning benchmarks.} We report pass@1 and pass@64 (\%) for each method on Qwen3-1.7B-Base and Qwen3-4B-Base. The best per column is in bold.}
\label{tab:main-results}
\setlength{\tabcolsep}{3pt}
% \small
\resizebox{\textwidth}{!}{%
\begin{tabular}{lcccccccccccc}
\toprule
\multicolumn{1}{c}{\textbf{Method}} & \multicolumn{2}{c}{\textbf{AIME 2025}} & \multicolumn{2}{c}{\textbf{BeyondAIME}} & \multicolumn{2}{c}{\textbf{HMMT 02/25}} & \multicolumn{2}{c}{\textbf{HMMT 02/26}} & \multicolumn{2}{c}{\textbf{MATH-500}} & \multicolumn{2}{c}{\textbf{Avg.}} \\
\cmidrule(lr){2-3} \cmidrule(lr){4-5} \cmidrule(lr){6-7} \cmidrule(lr){8-9} \cmidrule(lr){10-11} \cmidrule(lr){12-13}
 & pass@1 & pass@64 & pass@1 & pass@64 & pass@1 & pass@64 & pass@1 & pass@64 & pass@1 & pass@64 & pass@1 & pass@64 \\
\midrule
\multicolumn{13}{c}{\emph{Qwen3-1.7B-Base}} \\
\midrule
Base & 3.1 & 34.6 & 1.1 & 16.9 & 0.2 & 10.6 & 1.1 & 14.1 & 57.1 & 92.8 & 12.5 & 33.8 \\
GRPO & 4.8 & 29.3 & 1.1 & 15.6 & 0.5 & 10.0 & \textbf{3.8} & 15.5 & 70.5 & 92.6 & 16.1 & 32.6 \\
MaxRL & 6.7 & 36.5 & 1.3 & 21.8 & 1.3 & 18.4 & 3.2 & 17.7 & \textbf{72.1} & 94.0 & 16.9 & 37.7 \\
\rowcolor{gray!18} \textbf{CurveRL} & \textbf{7.7} & \textbf{40.9} & \textbf{1.4} & \textbf{23.6} & \textbf{1.7} & \textbf{19.3} & 3.0 & \textbf{21.0} & 71.6 & \textbf{94.6} & \textbf{17.1} & \textbf{39.9} \\
\midrule
\multicolumn{13}{c}{\emph{Qwen3-4B-Base}} \\
\midrule
Base & 6.9 & 44.3 & 4.0 & 23.8 & 1.0 & 22.5 & 3.0 & 20.6 & 67.5 & 95.7 & 16.5 & 41.4 \\
GRPO & 19.1 & 48.6 & 8.0 & 29.0 & 5.8 & 19.0 & 9.9 & 27.2 & 82.8 & 94.7 & 25.1 & 43.7 \\
MaxRL & 18.9 & 52.6 & 8.2 & 32.5 & 7.9 & 27.8 & 11.8 & 33.9 & \textbf{85.1} & 96.6 & 26.4 & 48.7 \\
\rowcolor{gray!18} \textbf{CurveRL} & \textbf{19.6} & \textbf{55.6} & \textbf{8.3} & \textbf{40.4} & \textbf{9.0} & \textbf{32.4} & \textbf{11.9} & \textbf{40.9} & 84.5 & \textbf{97.7} & \textbf{26.7} & \textbf{53.4} \\
\bottomrule
\end{tabular}%
}
\end{table}

\paragraph{Improved Trade-off Frontier.} \Cref{tab:main-results} reports evaluation results on five reasoning benchmarks, with similar results on the remaining three benchmarks deferred to Appendix~\ref{app:supp-table}. Across both Qwen3-1.7B-Base and Qwen3-4B-Base, CurveRL achieves the highest pass@$64$ on all benchmarks, indicating stronger performance on hard prompts. Compared with MaxRL, the strongest baseline, it improves average pass@$64$ by $\mathbf{+5.9\%}$ and $\mathbf{+9.7\%}$, respectively. Importantly, CurveRL also achieves a higher average pass@$1$ on both model sizes, indicating that its improvement in Best-of-$N$ performance does not come at the expense of single-shot accuracy. This mitigates the pass@$k$ degradation issue observed in prior work~\citep{yue2025does,chen2025pass}. Overall, our results show that the distribution-aware reweighting of CurveRL improves the Pareto frontier between pass@$1$ and pass@$k$ over the existing pointwise reweighting methods.

\begin{figure}[!t]
    \centering
    \includegraphics[width=1.0\linewidth]{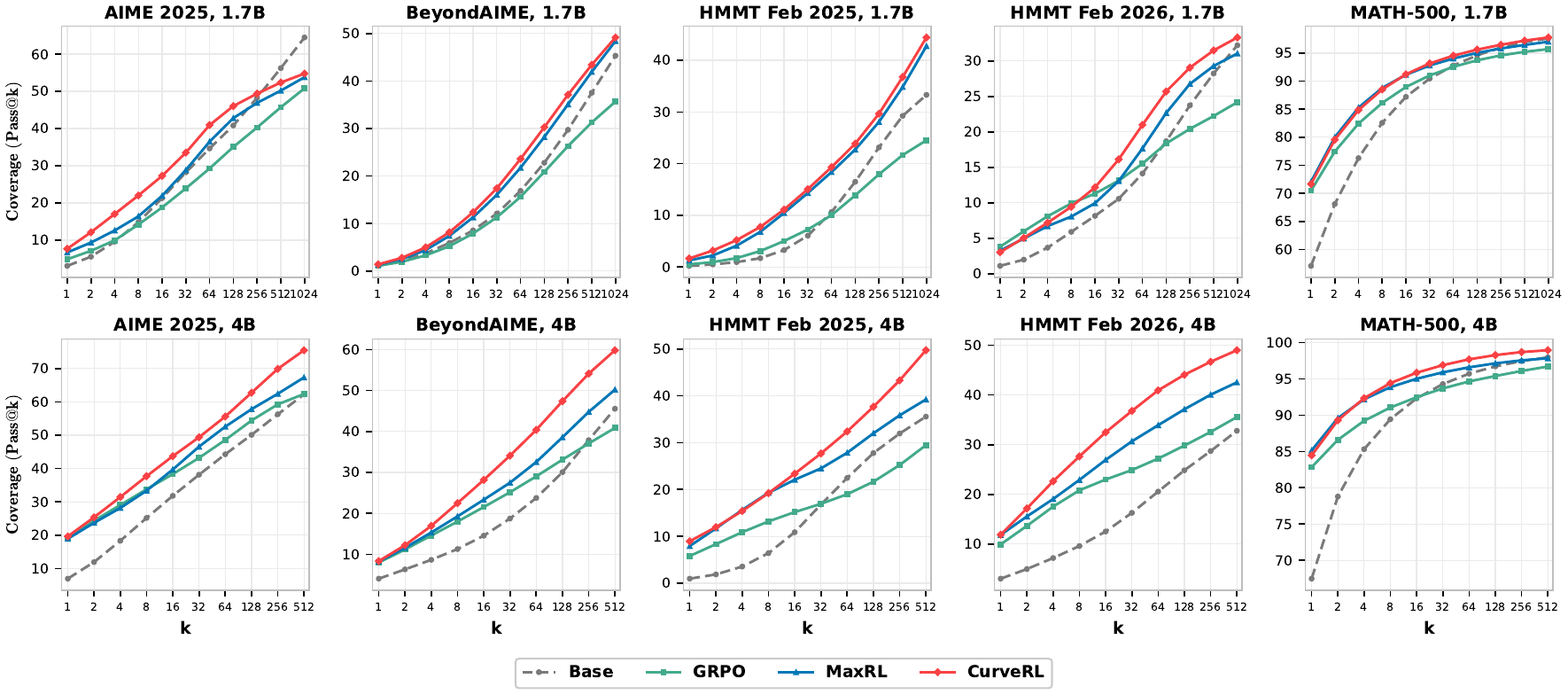}
    % \vspace{-5mm}
    \caption{ \textbf{Pass@$k$ scaling on five representative
    benchmarks.} Top row: Qwen3-1.7B-Base, $k\in\{1,\dots,1024\}$. Bottom
    row: Qwen3-4B-Base, $k\in\{1,\dots,512\}$. CurveRL outperforms GRPO and
    MaxRL across the full range of $k$ on both model sizes, and exceeds the pretrained base model on most panels.}
    \label{fig:main-passk}
    % \vspace{-5mm}
\end{figure}

\paragraph{Analysis of Pass@$k$ Scaling.} Following prior work~\citep{yue2025does,tajwar2026maximum,wu2025invisible}, we further examine the capability boundary of each algorithm through pass@$k$ scaling over a wide range of $k$. \Cref{fig:main-passk} reports results on five representative benchmarks, while the remaining three benchmarks in \Cref{fig:appendix-passk} of Appendix~\ref{app:supp-figure} follow similar trends. CurveRL consistently outperforms the other algorithms across a wide range of pass@$k$ values in most settings, indicating a higher upper bound on reasoning capability. Consistent with prior observations~\citep{yue2025does}, GRPO and MaxRL exhibit varying degrees of pass@$k$ degradation relative to the pretrained base model, whereas CurveRL exceeds the pretrained base model in $9$ out of $10$ panels of \Cref{fig:main-passk} across both model sizes and all values of $k$. Meanwhile, CurveRL’s advantage over the strongest baseline, MaxRL, grows as both the model scale and $k$ increase on challenging reasoning benchmarks, reaching approximately $\mathbf{+7.3\%}$ on HMMT Feb 2026 at $k=1024$ with Qwen3-1.7B-Base and $\mathbf{+26.8\%}$ on HMMT Feb 2025 at $k=512$ with Qwen3-4B-Base. This widening advantage suggests that CurveRL effectively broadens the search space of reasoning trajectories, enabling the policy to discover rare correct solutions. This property may further benefit both reward-free and reward-guided test-time scaling methods, including scalable Best-of-$N$ selection with strong statistical signals~\citep{kang2026scalable,fu2026deep}, lightweight probing~\citep{guo2026mining}, and process reward models (PRMs)~\citep{lightman2023let}.

\subsection{Mechanism Analysis of CurveRL}
\label{sec:analysis}

\paragraph{Difficulty Distribution Analysis.} We investigate how the algorithm in the post-training changes the RLVR model's capability over prompts with distinct difficulty levels. We partition prompts by empirical pass rate $\hat p$ over 2048 rollouts into four groups with different prior difficulty levels: \emph{unsolvable} ($\hat p=0$), \emph{hard} ($\hat p\in(0,1/2]$), \emph{medium} ($\hat p\in(1/2,1)$), and \emph{easy} ($\hat p=1$). As illustrated in \Cref{fig:difficulty-main} on two benchmarks BRUMO 2025 and MATH-500, CurveRL consistently reduces the fraction of \emph{unsolvable} prompts, especially on BRUMO 2025. This verifies that CurveRL has a larger potential to expand the real solvability boundary by making more unsolvable problems tractable.

% \begin{wrapfigure}{r}{0.6\textwidth}
\begin{figure}[htbp]
    \centering
    \vspace{-2mm}
    \includegraphics[width=0.8\linewidth]{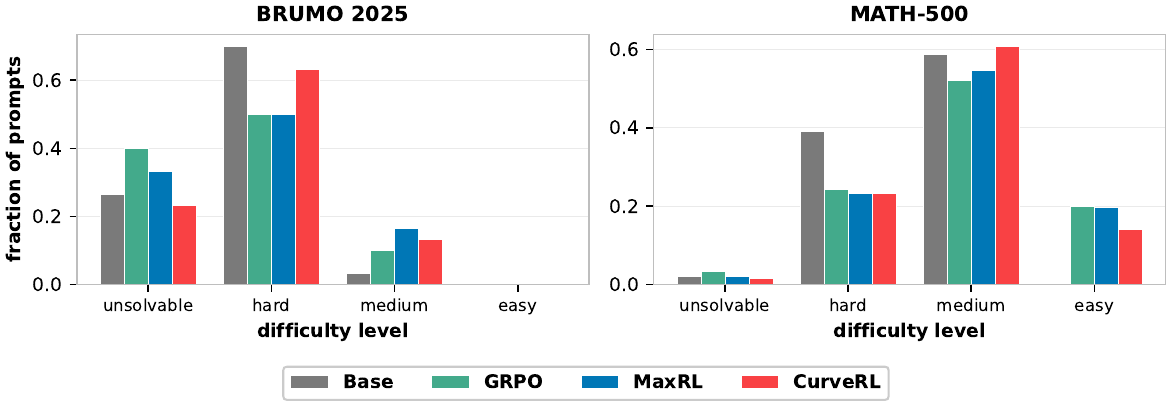}
    % \vspace{-5mm}
    \caption{Prompt Distribution across Difficulty in Qwen3-1.7B-Base.}
    \label{fig:difficulty-main}
    % \vspace{-6mm}
\end{figure}
% \end{wrapfigure}

\begin{wraptable}{r}{0.45\textwidth}
    \centering
    \vspace{-0mm}
    \caption{Majority@$2048$ on Qwen3-1.7B-Base.}
    \label{tab:majority-2048}
    % \small
    % \setlength{\tabcolsep}{4pt}
    \scalebox{0.95}{
\begin{tabular}{ccc}
    \toprule
    Method & BRUMO 2025 & MATH-500 \\
    \midrule
    Base   & 13.3 & 70.0 \\
    GRPO   & 16.7 & 74.6 \\
    MaxRL  & 16.7 & 77.2 \\
    \rowcolor{gray!18} \textbf{CurveRL} & \textbf{23.3} & \textbf{78.8} \\
    \bottomrule
    \end{tabular}
    }
    \vspace{-2mm}
\end{wraptable}
\noindent Moreover, CurveRL does not merely convert unsolvable prompts into barely solvable ones. It shifts a larger fraction of prompts into the \emph{medium} regime, where $\hat p>1/2$ and majority voting~\citep{wang2022self} becomes statistically reliable: repeated sampling can amplify individual correctness without requiring additional reward signals. This explains the stronger reward-free test-time scaling of CurveRL, as reflected by the higher majority@$2048$ scores on BRUMO 2025 and MATH-500
in \Cref{tab:majority-2048}.

\paragraph{Distribution-Aware and Data-Driven Weighting.} \Cref{fig:weight-comparison} visualizes how CurveRL differs from pointwise weighting algorithms from the perspective of prompt weights in the pass-rate space. 
\begin{enumerate}[leftmargin=*]
    \item The left panel suggests the evolution of the empirical reference density $\hat f_{\mathrm{ref}}(\hat p)$ during training: the pass-rate distribution shifts from a sharp concentration near the lowest pass-rate bins toward higher-pass-rate regions, becoming smoother and more spread. This provides empirical evidence for the distribution-transport view in Eq.~\eqref{eq:distribution_transport} of \Cref{sec:optimality_utility}. 
    \item The center panel showcases the resulting weight dynamics $w_t(\hat p)=\hat f_{\mathrm{ref}}(\hat p)/\hat F_{\mathrm{ref}}(\hat p)$ in CurveRL, which adapts to the evolving estimates of reference distribution $F_{\mathrm{ref}}$ and $f_{\mathrm{ref}}$. The CDF denominator preserves an emphasis on lower-quantile prompts, while the density numerator makes the allocation data-driven by assigning weights to pass-rate regions that are populated under the current policy. 

    \item The right panel contrasts this adaptive behavior with the static pointwise weights of GRPO and MaxRL, which can also be interpreted as a fixed prior in \Cref{thm:distribution}. With a symmetric weight $1/\sqrt{\hat{p}(1-\hat{p})}$, GRPO assigns large weights to both low- and high-pass-rate bins, while MaxRL employs $1/\hat{p}$ and monotonically emphasizes prompts with smaller pass rates.
\end{enumerate}

\begin{figure}[htbp]
    \centering
    \includegraphics[width=1.0\linewidth]{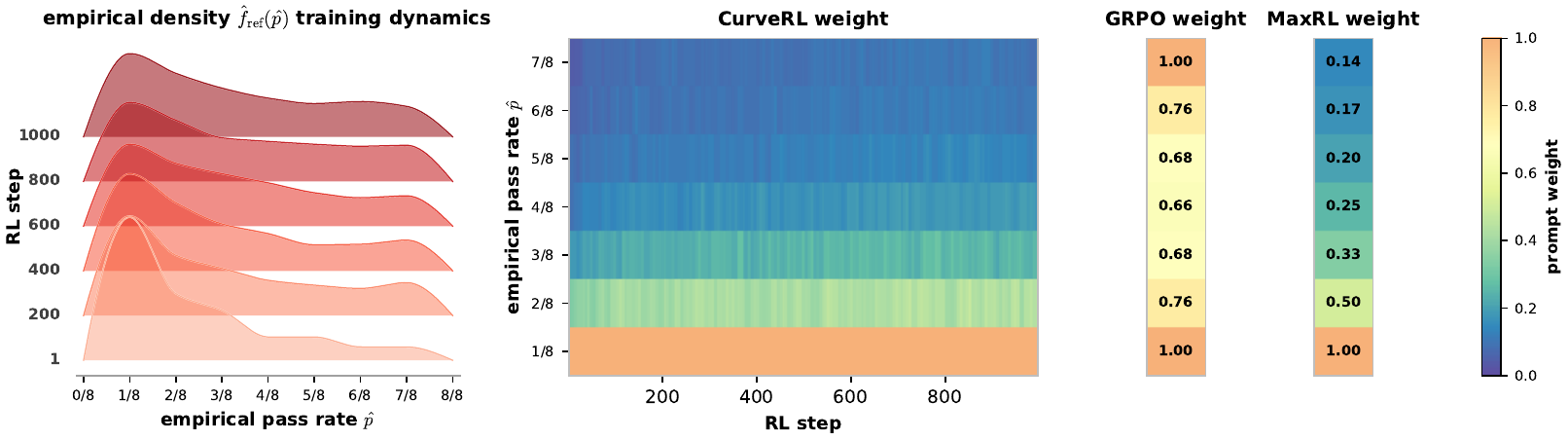}
    % \vspace{-3mm}
    \caption{\textbf{Distribution-aware and data-driven weighting of CurveRL on Qwen3-4B-Base.}
\emph{Left}: dynamics of CurveRL's empirical pass-rate density $\hat f_{\mathrm{ref}}(\hat p)$, smoothed by interpolation. \emph{Center}: dynamics of CurveRL's adaptive prompt weight $w_t(\hat p)=\hat f_{\mathrm{ref}}(\hat p)/\hat F_{\mathrm{ref}}(\hat p)$ . \emph{Right}: static GRPO weight $1/\sqrt{\hat p(1-\hat p)}$ and MaxRL weight $1/\hat p$ rescaled to $[0,1]$.}
% \vspace{-6mm}
    \label{fig:weight-comparison}
\end{figure}

\paragraph{Sensitivity Analysis of Sliding Window Size $t_0$.}
\label{sec:ablation} We further compare our selected sliding-window size $t_0=10$ batches in Algorithm~\ref{alg:curverl} with two extreme choices: a single-batch window ($t_0=1$) and a much wider window ($t_0=50$). \Cref{tab:sensitivity-t0} suggests that CurveRL with $t_0=10$ achieves a better overall trade-off, obtaining the largest pass@$64$ on all benchmarks while remaining competitive on pass@$1$ relative to the alternatives. Specifically, a smaller window size can introduce larger variance when estimating the pass-rate distribution, while a larger window size relies more heavily on older policies and thus introduces bias in estimating the current pass-rate distributions. Notably, even with $t_0=1$ or $50$, CurveRL still matches or exceeds the strongest pointwise baseline, MaxRL, on most benchmarks. Pass@$k$ scaling curves for CurveRL under the two $t_0$ are provided in \Cref{fig:sensitivity-passk} in Appendix~\ref{app:sensitivity-t0}.

\begin{table}[t!]
\centering
% \vspace{-4mm}
\caption{\textbf{Sensitivity to sliding window size $t_0$} regarding the pass@1 and pass@64 (\%) performance on Qwen3-1.7B-Base.}
\label{tab:sensitivity-t0}
\scalebox{0.72}{
\begin{tabular}{lcccccccccccc}
\toprule
\multicolumn{1}{c}{\textbf{$t_0$}} & \multicolumn{2}{c}{\textbf{BRUMO 2025}} & \multicolumn{2}{c}{\textbf{AIME 2025}} & \multicolumn{2}{c}{\textbf{HMMT 11/25}} & \multicolumn{2}{c}{\textbf{HMMT 02/26}} & \multicolumn{2}{c}{\textbf{MATH-500}} & \multicolumn{2}{c}{\textbf{Avg.}} \\
\cmidrule(lr){2-3} \cmidrule(lr){4-5} \cmidrule(lr){6-7} \cmidrule(lr){8-9} \cmidrule(lr){10-11} \cmidrule(lr){12-13}
 & pass@1 & pass@64 & pass@1 & pass@64 & pass@1 & pass@64 & pass@1 & pass@64 & pass@1 & pass@64 & pass@1 & pass@64 \\
\midrule
$1$ & 16.9 & 44.9 & 7.1 & 36.1 & 3.2 & 18.3 & \textbf{3.5} & 18.1 & \textbf{72.0} & 94.2 & 20.5 & 42.3 \\
\rowcolor{gray!18} $10$ (default) & \textbf{18.0} & \textbf{45.3} & \textbf{7.7} & \textbf{40.9} & \textbf{3.9} & \textbf{20.6} & 3.0 & \textbf{21.0} & 71.6 & \textbf{94.6} & \textbf{20.8} & \textbf{44.5} \\
$50$ & 15.8 & 43.3 & 6.0 & 37.5 & 3.6 & 19.9 & 2.5 & 17.5 & 69.3 & 94.1 & 19.4 & 42.5 \\
\bottomrule
\end{tabular}%
}
% \vspace{-3mm}
\end{table}

\section{Related Work}\label{app:relatedwork}

\paragraph{Action-level Exploration in Standard RL vs Context-level Distribution Control in RLVR.} In standard RL, the state-visitation distribution is often reshaped indirectly through: (i) importance weighting in off-policy correction~\citep{sutton1998reinforcement,mahmood2015emphatic}, (ii) experiment-design-style transition selection~\citep{mehta2021experimental,blau2022optimizing}, (iii) goal-conditioning~\citep{liu2022goal,eysenbach2022contrastive}, (iv) replay variants such as prioritized experience replay~\citep{lin1992self,schaul2015prioritized,andrychowicz2017hindsight,kapturowski2018recurrent,pan2022understanding}, all of which act on \emph{previously collected} data. RLVR with prompt reweighting~\citep{shao2024deepseekmath,guo2025deepseek,yu2025dapo,tajwar2026maximum,xiong2025reinforce,parashar2025curriculum} instead manipulates the \emph{online} prompt distribution directly. Action-level exploration becomes more challenging in this setting, as the trajectory-level action space is combinatorially large and the contextual-bandit structure provides no state-visitation dynamics to leverage~\citep{cui2025entropy,yue2025does,dai2025cde}. Therefore, the direct context distribution control by prompt reweighting supplies a new axis of information acquisition, which is exactly formulated by the PRCB framework we develop in Section~\ref{sec:optimality_framework}.

\paragraph{Existing Strategies of Context-level Distribution Control in RLVR.} There are mainly three classes of methods that exploit the freedom to shape the prompt distribution in RLVR. Our work casts all three mechanisms inside a single \textit{effective context distribution} $d_\theta(x)\propto d_0(x)w_\theta(x)$ and supplies the optimality principle for $w_\theta$ that has been missing from this line of work.
\begin{itemize}[leftmargin=8mm]
    \item \emph{Sample selection and dynamic filtering}: \cite{yu2025dapo,mao2026dynamics} discard prompts whose empirical pass rate equals $0$ or $1$, serving as a heuristic yet effective way to select the most learnable prompts. \cite{xiong2025reinforce} focuses on adaptively adjusting the number of samples for each prompt and prioritizes the harder prompts based on the prompt reweighting framework with non-linear RL objectives. We instead view reweighting as a direct context distribution control mechanism to improve the policy improvement by shaping learning dynamics, and establish the reweighting optimality from the perspective of utility functionals, which are risk-dependent. 

    \item \emph{Curriculum strategies}: Curriculum learning in RLVR, rooted in classical curriculum learning~\citep{bengio2009curriculum,narvekar2020curriculum}, schedules prompts from easy to hard~\citep{parashar2025curriculum,rajaraman2026learning,chen2025self}. As analyzed in Appendix~\ref{app:discussion_curriculum}, curriculum learning relies on an exogenous data schedule to help the policy optimization to chase the moving learnable window, while our method transforms the learnable window into a stationary and nearly fixed one by quantile reparameterization, thus enhancing the learning dynamics.

    \item \emph{Prompt reweighting}: \cite{davis2025objective} first formalizes the prompt weighting objective for LLM reasoning. This motivates MaxRL~\citep{tajwar2026maximum} by introducing the maximum likelihood principle into the RL objective and the adaptive sampling framework of \cite{xiong2025reinforce} through the lens of adaptive sample selection. This prompt reweighting framework is general to accommodate a flurry of RLVR algorithms, including the pass@K optimization in \cite{walder2025pass}. From the perspective of the model's edge of competence, \cite{zhang2025interplay,huang2026learning} also show that selecting prompts or manipulating the training distribution is crucial for effective learning. Our study focuses on this line of work in context distribution control of RLVR, as it allows more principled analysis than sample selection and curriculum learning, which can also be considered as special prompt reweighting in general.
\end{itemize}

\paragraph{Theoretical understanding of GRPO and Its Variants.} A growing GRPO family modifies the algorithm primarily through the denominator and group-level normalization: Dr.GRPO~\citep{liu2025understanding} removes the standard-deviation factor to debias the gradient; GVPO~\citep{zhang2025gvpo}, GPG~\citep{chu2025gpg}, and follow-ups~\citep{ge2026grpo,yang2026your,chen2025beyond} explore alternative group-level rescalings. These modifications are typically motivated by variance reduction or empirical fixes, leaving the form of the denominator unjustified at the population level. A separate strand analyzes GRPO and REINFORCE through the lens of U-statistics~\citep{zhou2026demystifying} and optimization dynamics~\citep{suk2026Optimization,mroueh2025reinforcement,huang2026learning}. In particular, \cite{huang2026learning} identifies a \emph{replay effect} when the difficulty spectrum of $p_\theta(x)$ is sufficiently smooth. MaxRL~\citep{tajwar2026maximum} imports the maximum likelihood principle into the policy gradient, which seemingly entails the optimality of maximum likelihood in statistics. However, we argue this is essentially a coincidence: maximum likelihood enjoys optimality under certain regularity conditions in statistical estimation, whereas policy optimization is mode-seeking instead of statistical estimation, where the optimality of $\log p_\theta$ does not transfer in a straightforward way. Our utility-functional view recovers MaxRL as one specific pointwise utility with a risk-seeking preference, and broadly provides a principled framework to explain the optimal prompt reweighting based on the choice of utility functional.

\paragraph{Distributional Structure in RLVR.}  The distributional structure between prompts and their rewards is crucial in algorithm analysis and design. \cite{barakat2026pass,walder2025pass} reveals the conflict between improving  pass@$K$ and pass@$1$ induced by the prompt interference, which is also supported by the fundamental property of deep nets called \textit{local elasticity}~\citep{he2019local,su2024envisioning}. This trade-off also mirrors the well-known tension between adversarial robustness and clean accuracy~\citep{tsipras2018robustness,zhang2019theoretically}: pushing harder prompts (akin to adversarial examples) trades off accuracy on easier ones. Therefore, it is beneficial to include the correlations between prompts in the algorithm design instead of treating them independently. Our distribution-aware method emerges across the prompt distribution in the pass-rate space by leveraging the rank/quantile and density information of $p_\theta(x)$, offering an explicit knob for navigating this trade-off between hard and easy prompts. \cite{wu2026Quantile} proposes quantile advantage estimation by replacing the mean with a group-wise $K$-quantile baseline to address entropy collapse and entropy explosion, which provides further evidence to demonstrate that leveraging the distributional structure of prompts can be beneficial.

\paragraph{Optimality of Utility Functional and Risk-sensitive Control.}  Risk-sensitive RL~\citep{howard1972risk,mihatsch2002risk,majumdar2017risk,wang2022risk,smith2023exponential,bauerle2024markov} replaces the expectation of cumulative rewards with other utility functionals or risk measures, reflecting the principle that there is no universal optimality in this choice, which mainly depends on the designer's risk attitude. Distributional RL~\citep{dabney2018implicit,bellemare2023distributional} makes this dependence explicit by allowing arbitrary distortion functions. Our distribution-aware utility instead consumes the distributional structure of the pass rates, not its absolute magnitude. The CDF/quantile coordinate transform with the associated distortion function connects directly to dual utility and spectral risk measures~\citep{yaari1987dual,wang1996premium,acerbi2002spectral,balbas2009properties}, which provides the risk-control rationale behind the quantile coordinate transform of Section~\ref{sec:distribution}.

% \clearpage
\section{Discussion and Conclusion}\label{sec:discussion}

Our study suggests that the prompt reweighting in RLVR should be understood as principled context distribution control rather than heuristic modifications of weights in policy gradients. Under this view, there is no universally optimal
prompt weight independent of the training objective. Instead, the optimal weight is utility-dependent and determined by the marginal value of improving each prompt's pass rate, which is formulated by the functional derivative. By using the quantile coordinate transform, CurveRL instantiates context distribution control with a distribution-aware utility functional in the pass-rate quantile space, allowing the prompt measure to adapt to the model's current competence profile through the distributional structure of pass rates independent of existing pointwise rules.
% More broadly, our results highlight that the effective RLVR requires a principled context distribution control as a new axis of information acquisition in RL.

\paragraph{Limitations and Future Work.} A natural extension of our method is to integrate pointwise and distribution-aware utility functionals into more expressive prompt weighting schemes. We explored two such integrated strategies in Appendix~\ref{app:discussion_integration}, but did not observe empirical improvements. This suggests that the two utility functional classes may induce different geometries in the pass-rate space, and simply adding or multiplying them may not reconcile their optimization effects. A rigorous characterization of these geometries and their impact on learning dynamics remains an important direction for future work. Moreover, defining weight optimality directly in the prompt space, rather than in the pass-rate space, additionally captures the effect of parameter updates, although computing functional derivatives over high-dimensional prompt representations is generally intractable.
% Lastly, exploring the optimal reweighting in on-policy distillation regime is also an intriguing direction. 

\section*{Acknowledgments}

This work was supported in part by NIH grants R01EB036016, R01EB037101, and R01MH143267, NSF grant DMS-2310679, a Meta Faculty Research Award, and Wharton AI for Business. The content is solely the responsibility of the authors and does not necessarily represent the official views of the NIH.

\bibliographystyle{plainnat}
\bibliography{reference}

\clearpage
\appendix

%%%%%%%%%%%% set table of content for appendix (start)
\addcontentsline{toc}{section}{Appendix} % Add the appendix text to the document TOC
\part{Appendix} % Start the appendix part
\parttoc % Insert the appendix TOC
%%%%%%%%%%% set table of content for appendix (end)
\clearpage

\section{Theoretical Results}\label{app:theory}

\subsection{Functional derivative under Pointwise Utility Reduces to Partial Derivative}\label{app:theory_reduction}

If the utility function $\mathcal{U}_\theta$ equals to the pointwise utility function, functional derivative degenerates to the partial derivative, i.e., 
\begin{align*}
    \mathcal{U}_\theta =\int_{\mathcal{X}} g\left(p_\theta(x)\right) d_0(x) d x \Rightarrow \frac{\delta \mathcal{U}_\theta}{\delta p_\theta}(x)=g^{\prime}\left(p_\theta(x)\right)=\frac{\partial g\left(p_\theta(x)\right)}{\partial p_\theta(x)}.
\end{align*}
Even though it is a classical theoretical result in Calculus of Variations~\citep{gelfand2000calculus} and Partial Differential Equation~\citep{evans2022partial}, we provide a short version of the proof as a reference for completeness, as the remaining theoretical results in our paper all rely on this backbone.

\begin{proof}

Firstly, after the perturbation, the perturbed utility function is
\begin{align*}
    \mathcal{U}_\theta\left(p_\theta+\epsilon h\right)=\int_{\mathcal{X}} g\left(p_\theta(x)+\epsilon h(x)\right) d_0(x) d x.
\end{align*}
By Taylor expansion at $t=p_\theta(x)$, we have
\begin{align*}
    g\left(p_\theta(x)+\epsilon h(x)\right)=g\left(p_\theta(x)\right)+\epsilon \cdot g^{\prime}\left(p_\theta(x)\right) \cdot h(x)+r_\epsilon(x),
\end{align*}
where the reminder term satisfies $r_\epsilon(x)/\epsilon \rightarrow 0$ pointwise as $\epsilon \rightarrow 0$. Therefore, the utility difference is
\begin{align*}
    \mathcal{U}_\theta\left(p_\theta+\epsilon h\right)-\mathcal{U}_\theta\left(p_\theta\right)=\epsilon \int_{\mathcal{X}} g^{\prime}\left(p_\theta(x)\right) h(x) d_0(x) d x+\int_{\mathcal{X}} r_\epsilon(x) d_0(x) d x.
\end{align*}
Next, by regularity condition, including the dominated convergence, the remainder term vanishes. Then we have
\begin{align*}
    \frac{\mathcal{U}_\theta\left(p_\theta+\epsilon h\right)-\mathcal{U}_\theta\left(p_\theta\right)}{\epsilon}&=\int_{\mathcal{X}} g^{\prime}\left(p_\theta(x)\right) h(x) d_0(x) d x+\int_{\mathcal{X}} \frac{r_\epsilon(x)}{\epsilon} d_0(x) d x \\
    & \rightarrow \int_{\mathcal{X}} g^{\prime}\left(p_\theta(x)\right) h(x) d_0(x) d x \quad \text{as} \ \epsilon \rightarrow 0.
\end{align*}
By Definition~\ref{def:weight} and the uniqueness of Riesz representation, we obtain that
\begin{align*}
    \int_{\mathcal{X}} \frac{\delta \mathcal{U}_\theta}{\delta p_\theta}(x) \cdot h(x) d_0(x) d x=\int_{\mathcal{X}} g^{\prime}\left(p_\theta(x)\right) \cdot h(x) d_0(x) d x, \forall h \in L^2(d_0).
\end{align*}
Consequently, it results in the conclusion that $\frac{\delta \mathcal{U}_\theta}{\delta p_\theta}(x)=g^{\prime}\left(p_\theta(x)\right)$ for $d_0$ almost everywhere for $x$.

\end{proof}

\subsection{Proof of Proposition~\ref{prop:entropic_risk}}\label{app:theory_entropicrisk}

\textbf{Proposition}~\ref{prop:entropic_risk} (Entropic Risk RL Interpolates Classical RL and MaxRL.) Assume $p_\theta(x) > 0$ for $d_0$-almost every $x$ and $\frac{\left\|\nabla_\theta p_\theta(x)\right\|}{p_\theta(x)} \in L^1\left(d_0\right)$, then 
\begin{align}
      \lim _{\eta \rightarrow 0^+} \nabla_\theta \mathcal{U}_\theta^{\mathrm{risk}}(\eta) = \nabla_\theta J_{\mathrm{RL}}(\theta), \ \lim _{\eta \rightarrow \infty} \eta \nabla_\theta \mathcal{U}_\theta^{\mathrm{risk}}(\eta) =\nabla_\theta J_{\mathrm{ML}}(\theta).
\end{align}

\begin{proof}
    As the reward $r(x, y)$ is a Bernoulli random variable, we have
    \begin{align*}
        \mathbb{E}_{y \sim \pi_\theta(\cdot \mid x)} e^{\eta r(x, y)}=\left(1-p_\theta(x)\right) e^0+p_\theta(x) e^\eta=1-p_\theta(x)+p_\theta(x) e^\eta.
    \end{align*}
    Therefore, we have:
    \begin{align*}
        \mathcal{U}_\theta^{\mathrm{risk}}(\eta) =\mathbb{E}_{x \sim d_0}\left[\frac{1}{\eta} \log \left(1-p_\theta(x)+p_\theta(x) e^\eta\right)\right]:=\mathbb{E}_{x \sim d_0}\left[ h_\eta(p_\theta(x)) \right],
    \end{align*}
    where we denote the weight function as $w_\theta^\eta(x)$ and therefore we have
    \begin{align*}
         \nabla_\theta \mathcal{U}_\theta^{\mathrm{risk}}(\eta) :=\mathbb{E}_{x \sim d_0}\left[ w_\theta^\eta(x) \nabla_\theta p_\theta(x)\right] = \mathbb{E}_{x \sim d_0}\left[\frac{e^\eta-1}{\eta\left[1+\left(e^\eta-1\right) p_\theta(x)\right]} \nabla_\theta p_\theta(x)\right].
    \end{align*}
    As $\eta > 0$, the weight function $w^\eta_\theta(x)$ above is always positive and thus $h_\eta$ is an increasing function of $p_\theta(x)$, which leads to a meaningful utility function. Next, we discuss the two limiting behaviors of $ \nabla_\theta \mathcal{U}_\theta^{\mathrm{risk}}(\eta).$

    \paragraph{Case 1: $\eta \rightarrow 0^+$.} By Taylor expansion, $e^\eta = 1 + \eta + o(\eta)$ as $\eta \rightarrow 0^+$. Consequently, we have the pointwise limit:
        \begin{align*}
            \frac{e^\eta-1}{\eta\left[1+\left(e^\eta-1\right) p_\theta(x)\right]} = \frac{\eta + o(\eta)}{\eta\left[1+\left(\eta + o(\eta)\right) p_\theta(x)\right]} = \frac{1+o(1)}{1+\eta p_\theta(x) + o(\eta)p_\theta(x)} \rightarrow 1, \ (\eta \rightarrow 0^+).
        \end{align*}
    Since $w_\eta(p_\theta(x)) \to 1$ and is uniformly bounded,
the limit can be exchanged with expectation. This implies that $ \lim _{\eta \rightarrow 0^+} \nabla_\theta \mathcal{U}_\theta^{\mathrm{risk}}(\eta)  = \mathbb{E}_{x \sim d_0}\left[ 1 \cdot \nabla_\theta p_\theta(x)\right] = \nabla_\theta J_{\mathrm{RL}}(\theta)$.
        
    \paragraph{Case 2: $\eta \rightarrow +\infty$.} Denote $a_\eta \sim b_\eta$ when $\eta \rightarrow +\infty$ if $\frac{a_\eta}{b_\eta}\rightarrow1$. Firstly, given each prompt $x$, we have the pointwise limit:
    \begin{align*}
        \eta w_\theta^\eta(x) =  \frac{e^\eta-1}{1+\left(e^\eta-1\right) p_\theta(x)} = \frac{1}{\frac{1}{e^\eta -1} + p_\theta(x)} \rightarrow \frac{1}{p_\theta(x)}. \ (\eta \rightarrow +\infty)
    \end{align*}
    This implies $w_\theta^\eta(x) \sim \frac{1}{\eta p_\theta(x)}$ as $\eta \rightarrow +\infty$. However, the pointwise convergence does not directly ensure the limit of the integral is the integral of the limit, i.e., $\lim \mathbb{E}=\mathbb{E} \lim$, which additionally requires some mild conditions~(i.e., Dominated convergence theorem~\citep{evans2025measure}). Since $\eta w_\theta^\eta(x)  = \frac{1}{\frac{1}{e^\eta -1} + p_\theta(x)} \leq \frac{1}{p_\theta(x)}$, according to our assumption in Proposition~\ref{prop:entropic_risk}, we have
    \begin{align*}
        \left|\eta w_\theta^\eta \left(p_\theta(x)\right) \nabla_\theta p_\theta(x)\right| \leq \frac{\left|\nabla_\theta p_\theta(x)\right|}{p_\theta(x)} \in L^1\left(d_0\right),
    \end{align*}
    where the function $g(x) \in L^1\left(d_0\right)$ indicates that $\mathbb{E}_{x \sim d_0}[|g(x)|]<\infty$. This inequality implies that a sequence of functions is bounded in absolute value by an integrable function. Therefore, the dominated convergence theorem~\citep{evans2025measure} yields
    \begin{align*}
        \lim_{\eta \rightarrow \infty} \eta \nabla_\theta \mathcal{U}_\theta^{\mathrm{risk}}(\eta) =\lim _{\eta \rightarrow \infty} \mathbb{E}_{x \sim d_0}\left[\eta w_\theta^\eta(x) \nabla_\theta p_\theta(x)\right]  = \mathbb{E}_{x \sim d_0}\left[\frac{1}{p_\theta(x)} \nabla_\theta p_\theta(x)\right]=\nabla_\theta J_{\mathrm{ML}}(\theta).
    \end{align*}
\end{proof}

\subsection{Proof of Proposition~\ref{prop:CDF_bound}}\label{app:theory_CDFbound}
\noindent \textbf{Proposition~\ref{prop:CDF_bound}}. Denote $W_1$ as 1-Wasserstein distance. If $\psi$ is $L_\psi$-Lipschitz and $\Vert f_\theta \Vert_\infty < \infty$, then
    \begin{align}
        \left| \mathcal{U}_\theta(F_{\text{ref}}) - \mathcal{U}_\theta(F_\theta) \right| \leq L_\psi \Vert f_\theta \Vert_\infty W_1(\mu_{\text{ref}}, \mu_\theta).
    \end{align}
\begin{proof}
    As $\psi$ is $L_\psi$-Lipschitz, we have $\left|\psi(a) - \psi(b) \right| \leq L_\psi \left| a - b \right|$. $\Vert f_\theta \Vert_\infty < \infty$ implies a bounded density function or a bounded essential supremum of the density, i.e., $f_\theta(t) \leq \Vert f_\theta \Vert_\infty $ almost everywhere (a.e.), indicating that the inequality holds except on a set of measure zero. Recall the definition of 1-Wasserstein distance $W_1(\mu, \nu)$ to characterize the difference of two measure $\mu$ and $\nu$:
    \begin{align*}
       W_1(\mu, \nu) =  \inf _{\gamma \in \Pi(\mu, \nu)} \int\|x-y\|_1 d \gamma(x, y) = \int_{\mathbb{R}}\left|F_\mu(t)-F_\nu(t)\right| d t.
    \end{align*}
where $\gamma$ is the coupling or the transport plan, and $\Pi(\mu, \nu)$ is the joint distribution space with the marginal distributions as $\mu$ and $\nu$. The RHS above is the equivalent form of 1-Wasserstein distance, which is pivotal for our upper bound. In addition, we also write down the two utility functions:
\begin{align*}
    \mathcal{U}_\theta(F_{\text{ref}}) &= \mathbb{E}_{x\sim d_0}\left[\psi\left(F_{\mathrm{ref}}\left(p_\theta(x)\right)\right)\right]=\int \psi(F_{\mathrm{ref}}(t)) d \mu_\theta(t) \\
    \mathcal{U}_\theta(F_\theta) & = \mathbb{E}_{x\sim d_0}\left[\psi\left(F_{\theta}\left(p_\theta(x)\right)\right)\right] = \int \psi(F_{\theta}(t)) d \mu_\theta(t) =\int_0^1 \psi(z) d z = \text{const}.
\end{align*}
Putting all together, we can derive
\begin{align*}
    \left| \mathcal{U}_\theta(F_{\text{ref}}) - \mathcal{U}_\theta(F_\theta)\right| & = \int \left[ \psi(F_{\mathrm{ref}}(t)) - \psi(F_{\theta}(t)) \right] d \mu_\theta(t)  \\
    & \leq \int \left| \psi(F_{\mathrm{ref}}(t)) - \psi(F_{\theta}(t)) \right| d \mu_\theta(t) \\
    & \overset{(a)}{\leq} L_\psi \int \left| F_{\mathrm{ref}}(t) - F_{\theta}(t) \right| d \mu_\theta(t)\\
    & = L_\psi \int \left| F_{\mathrm{ref}}(t) - F_{\theta}(t) \right| f_\theta(t) d t \\
    & \overset{(b)}{\leq} L_\psi \Vert f_\theta \Vert_\infty \int \left| F_{\mathrm{ref}}(t) - F_{\theta}(t) \right| d t \\
    & \overset{(c)}{=} L_\psi \Vert f_\theta \Vert_\infty W_1(\mu_{\text{ref}}, \mu_\theta),
\end{align*}
where $(a)$ holds by the $L_\psi$-Lipschitz condition of $\psi$, $(b)$ holds by $\Vert f_\theta \Vert_\infty < \infty$, and $(c)$ satisfies due to the definition of 1-Wasserstein distance. 

\paragraph{Interpretation.} The upper bound in Proposition~\ref{prop:CDF_bound} strengthens the connection of our distribution-aware utility and a distribution- or geometry-aware Wasserstein distance. Specifically, this bound shows that the discrepancy between the two utility functionals is controlled by the Wasserstein distance between the reference and current pass-rate distributions. In particular, it provides a formal justification that our method is distribution-aware: the utility depends on the global structure of the pass-rate distribution rather than pointwise values.

\end{proof}

\subsection{Proof of Theorem~\ref{thm:distribution}}\label{app:theory_distribution}

\noindent \textbf{Theorem~\ref{thm:distribution}}~(\textbf{Pointwise Weight Induces a Prior $F_{\mathrm{ref}}$}) Denote $p=p_\theta(x)$ and the pointwise weight $w_\theta(x)=g^\prime(p_\theta(x)):=w(p)$ in the pass-rate space. Assume $\int_p^1 w(t) dt < \infty$ for any $p \in (0, 1]$. Under the distortion $\psi(u)=\log u$ in Eq.~\eqref{eq:ours}, $w(p) = f_{\mathrm{ref}}\left(p\right)/F_{\mathrm{ref}}\left(p\right)$ admits a unique $F_\mathrm{ref}$:
\begin{align*}
    F_\mathrm{ref}(p) = \exp \left(-\int_p^1 w(t) dt \right).
\end{align*}

\begin{proof}

Recall that $w_\theta:(0,1] \rightarrow \mathbb{R}_{\geq 0}$ is a non-negative weight function that satisfies
\begin{align*}
    \int_p^1 w(t) d t<\infty.
\end{align*}
for any $p\in (0, 1]$. The connection between the existing pointwise weight $w_\theta(\cdot)$ and the distribution-aware weight $\frac{f_{\mathrm{ref}}\left(p\right)}{F_\mathrm{ref}\left(p\right)}$ requires:
\begin{align*}
    w(p) = \frac{f_{\mathrm{ref}}\left(p\right)}{F_\mathrm{ref}\left(p\right)}.
\end{align*}

\paragraph{Example 1: REINFORCE with $w(p)=1$.} This requires $(\log F_{\mathrm{ref}}(p))^\prime  = 1$ and $\log F_{\mathrm{ref}}(p) = p+C$, i.e., 
\begin{align*}
    F_{\mathrm{ref}}(p) = e^{p+C}
\end{align*}
Let $F_{\mathrm{ref}}(1)=1$, we have $C=-1$ and $f_{\mathrm{ref}} = e^{p-1}$. Notably, there is a point mass at $p=0$ as $F_{\mathrm{ref}}(0)=e^{-1}$. This corresponds to a reflected and truncated exponential distribution. Assume $Z\sim \text{Exp}(1)$  with $p_Z(z)=e^{-z}$ for $z \in [0, +\infty)$. Define $Y = 1 - Z$ such that $y \in (-\infty, 1]$. $P_Y(Y\leq y) = P (1-Z \leq y) = 1 - F_Z(1-y)=e^{y-1}$. Furthermore, by putting all mass in $(-\infty, 0)$ on the point $z=0$, the sub-distribution becomes $F_{\mathrm{ref}}(p)  = e^{p-1}, \ p\in [0,1]$, with the sub-density $f_{\mathrm{ref}}(p)=e^{p-1}$.

\paragraph{Example 2: GRPO with $w(p)=1/\sqrt{p(1-p)}$.} This requires $\left(\log F_{\mathrm{ref}}\right)^{\prime}(p)=1 / \sqrt{p(1-p)}$, implying
\begin{align*}
    \log F_{\mathrm{ref}}(p)=\int \frac{d p}{\sqrt{p(1-p)}}=2 \arcsin (\sqrt{p})+C \Rightarrow F_{\mathrm{ref}}(p) = \exp (2 \arcsin (\sqrt{p})+C).
\end{align*}
As $\arcsin(1)=\frac{\pi}{2}$, we let $F_{\mathrm{ref}}(1)=1$. Then $C = -\pi$. Consequently, we have
\begin{align*}
    F_{\mathrm{ref}}(p) = \exp (2 \arcsin (\sqrt{p})-\pi), \quad f_{\mathrm{ref}}(p) = \frac{\exp (2 \arcsin (\sqrt{p})-\pi)}{\sqrt{p(1-p)}}.
\end{align*}
Note that this distribution is not common one, but a normalized arcsine distribution after an exponential transformation.

\paragraph{Example 3: MaxRL with $w(p)=1/p$.} It requires $\left(\log F_{\mathrm{ref}}\right)^{\prime}(p)=1 / p$, which implies $\log F_{\mathrm{ref}}(p) = \log p + C$ and $F_{\mathrm{ref}}(p) = \exp(\log p+C)$. Let $F_{\mathrm{ref}}(1)=1$, we have $C=0$. Therefore, we have $F_{\mathrm{ref}}(p) = p$ and $f_{\mathrm{ref}}(p)=1$, which is a uniform distribution in $[0,1]$. We illustrate the density and CDF in Figure~\ref{fig:prior}.

\begin{figure}[hbtp]
    \centering
    \includegraphics[width=0.8\linewidth]{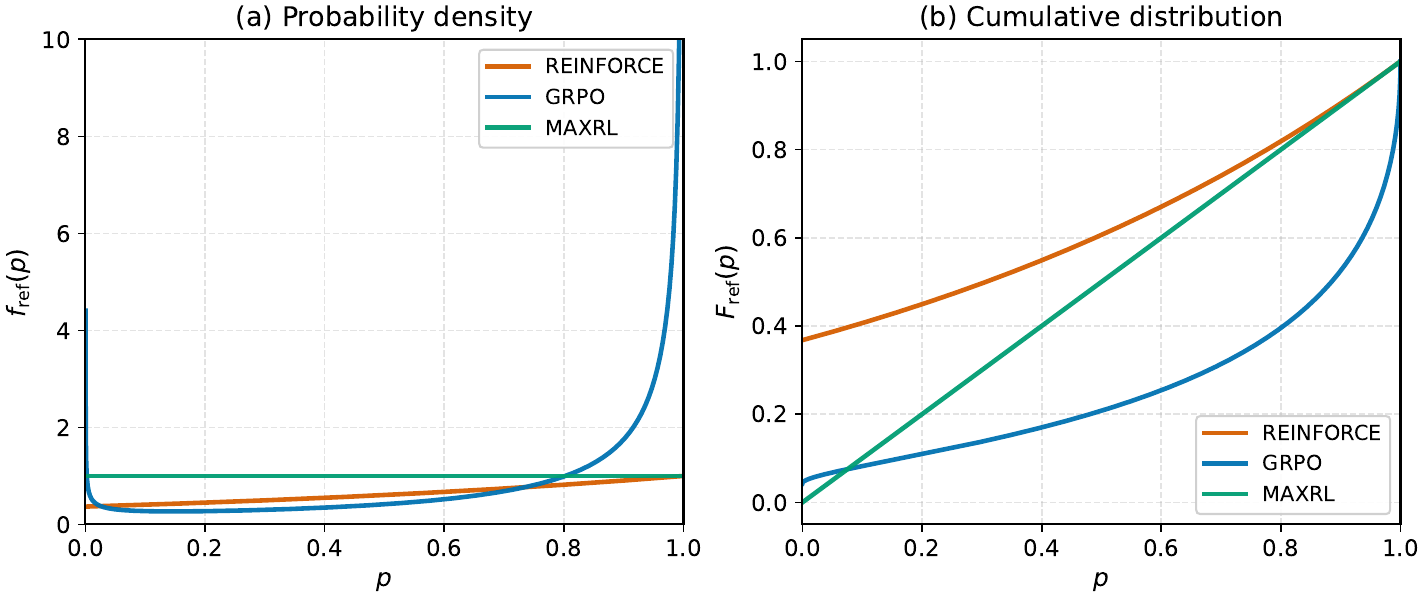}
    \caption{$f_{\mathrm{ref}}$ and $F_{\mathrm{ref}}$ for REINFORCE, GRPO, and MaxRL.}
    \label{fig:prior}
\end{figure}

\paragraph{General Case.} In general, there is a key connection equation:
\begin{align*}
    \int_p^1 w(t) dt  =  \int_p^1 \frac{f_{\text{ref}}\left(t\right)}{F_\text{ref}\left(t\right)} d t = \log F_\text{ref}(1) - \log F_\text{ref}(p) = - \log F_\text{ref}(p).
\end{align*}
This implies that 
\begin{align*}
    F_\text{ref}(p) = \exp \left(-\int_p^1 w(t) dt \right).
\end{align*}
Since $w$ is non-negative, we can easily verify that $F_\text{ref}$ is non-decreasing function as $F_\text{ref}(p_1) < F_\text{ref}(p_2)$ if $p_1 < p_2$ and $F_\text{ref}(1)=1$. Therefore, $F_\text{ref}$ is a valid CDF. Given the fact that $\frac{d }{dp} \left(-\int_p^1 w(t) dt\right)  = w(p)$, by the chain rule, we can derive
\begin{align*}
    f_\text{ref}(p) = \frac{d }{dp} \left(\exp \left(-\int_p^1 w(t) dt \right)\right)  =  F_\text{ref}(p) w(p),
\end{align*}
which is the exact goal we want to prove in the beginning. 

\paragraph{Uniqueness.} Assume $\tilde{F}$ also satisfies $\frac{\tilde{f}(p)}{\tilde{F}(p)}=w(p)$. This implies that
\begin{align*}
    \frac{d}{d p} \log \tilde{F}(p)=w(p). 
\end{align*}
By taking the integral between $[p, 1]$ on both sides, we derive
\begin{align*}
    \log \tilde{F}(1) - \log \tilde{F}(p) = \int_p^1 w(t) dt\Rightarrow   \log \tilde{F}(p) = -\int_p^1 w(t) dt \Rightarrow \tilde{F}(p) = \exp \left(-\int_p^1 w(t) dt \right).
\end{align*}
Therefore, we have $\tilde{F} = F_{\mathrm{ref}}$ and $F_{\mathrm{ref}}$ is unique. 

\paragraph{Remark.} The induced $F_{\mathrm{ref}}$ only depends on the pointwise weight $w(p)$ with the pass rate value $p$ in the pass-rate space, which is independent of the prompt $x\sim d_0$ and the policy parameter $\theta$.

\end{proof}

\subsection{Monotone Calibration Invariance Property}\label{app:theory_invariance}

For any increasing function $G: [0, 1]\rightarrow[0, 1]$ on the absolute value of pass rate $p_\theta(x)$, the gradient of our distribution-aware method keeps the same form, which we call a monotone calibration invariance property in Proposition~\ref{prop:invariance}:
\begin{proposition}[\textbf{Monotone Calibration Invariance}]\label{prop:invariance}  Assume $F_{\mathrm{ref}}=F_{\theta_{\text{old}}}$ and denote $\tilde{p}_\theta(x) = G(p_\theta(x))$, i.e., $\tilde{p}_\theta = G \circ p_\theta$. Denote the induced reference CDF and density as $\tilde{F}_{\mathrm{ref}}$ and $\tilde{f}_{\mathrm{ref}}$ , then
    \begin{align*}
        \nabla_\theta J_{\text{curve}}\left(\theta ; p_\theta, F_{\mathrm{ref }}\right)=\nabla_\theta J_{\text{curve }}\left(\theta; \tilde{p}_\theta, \tilde{F}_{\mathrm{ref}}\right).
    \end{align*}
\end{proposition}

\begin{proof}
    After the transformation with $\tilde{p}_\theta$, we have
\begin{align*}
    \tilde{F}_{\text{ref}}(u) = \tilde{F}_{\theta_{\text{old}}}(u) = \mathbb{P}_{x\sim d_0}(\tilde{p}_{\theta_{\text{old}}}(x) \leq u) = \mathbb{P}_{x\sim d_0}(p_{\theta_{\text{old}}}(x) \leq G^{-1}(u)) = F_{\mathrm{ref}}(G^{-1}(u)).
\end{align*}
Then, we can derive
\begin{align*}
   \tilde{f}_{\text{ref}}(u)  =  \frac{d \tilde{F}_{\text{ref}}(u)}{d u} = f_{\mathrm{ref}}(G^{-1}(u)) \frac{d G^{-1}(u)}{du}. % = f_{\mathrm{ref}}(G^{-1}(u)) \frac{1}{G'(u)},
\end{align*}
Based on the derivative equation of the inverse function and let $u = G(v)$, we know $(G^{-1}(u))' = \frac{1}{G'(v)}$. Therefore, we have
\begin{align*}
     \tilde{f}_{\text{ref}}(u)  = f_{\mathrm{ref}}(G^{-1}(u)) \frac{d G^{-1}(u)}{du} = f_{\mathrm{ref}}(G^{-1}(u)) \frac{1}{G'(G^{-1}(u))}.
\end{align*}

Recall the gradients before and after the transformation have the following form:
\begin{align*}
    \nabla_\theta J_{\text{curve}}(\theta;p_\theta, F_{\mathrm{ref}}) &  = \mathbb{E}_{x\sim d_0}\left[ \frac{f_{\text{ref}}\left(p_\theta(x)\right)}{F_\text{ref}\left(p_\theta(x)\right)} \nabla_\theta p_\theta(x)\right] \\
    \nabla_\theta J_{\text{curve}}(\theta; \tilde{p}_\theta, \tilde{F}_{\mathrm{ref}}) &  = \mathbb{E}_{x\sim d_0}\left[ \frac{\tilde{f}_{\text{ref}}\left(\tilde{p}_\theta(x)\right)}{\tilde{F}_\text{ref}\left(\tilde{p}_\theta(x)\right)} \nabla_\theta \tilde{p}_\theta(x)\right].
\end{align*}
Next, we know that $\nabla \tilde{p}_\theta(x) = G'(p_\theta(x)) \nabla p_\theta(x)$. By putting all together, we have the following key result:
\begin{align*}
\frac{\tilde{f}_{\text{ref}}\left(\tilde{p}_\theta(x)\right)}{\tilde{F}_\text{ref}\left(\tilde{p}_\theta(x)\right)} \nabla_\theta \tilde{p}_\theta(x) & = \frac{f_{\mathrm{ref}}(G^{-1}(\tilde{p}_\theta(x))) \frac{1}{G'(G^{-1}(\tilde{p}_\theta(x)))}}{F_{\mathrm{ref}}(G^{-1}(\tilde{p}_\theta(x)))}  \nabla_\theta \tilde{p}_\theta(x)  \\
& = \frac{f_{\mathrm{ref}}(G^{-1}(G(p_\theta(x)))) \frac{1}{G'(G^{-1}(G(p_\theta(x))))}}{F_{\mathrm{ref}}(G^{-1}(G(p_\theta(x))))}  \nabla_\theta \tilde{p}_\theta(x) \\
& = \frac{f_{\mathrm{ref}}(p_\theta(x)) \frac{1}{ G'(p_\theta(x))}}{F_{\mathrm{ref}}(p_\theta(x))}  G'(p_\theta(x)) \nabla p_\theta(x) \\
& = \frac{f_{\mathrm{ref}}(p_\theta(x))}{F_{\mathrm{ref}}(p_\theta(x))}  \nabla p_\theta(x).
\end{align*}
This implies that
\begin{align*}
      \nabla_\theta J_{\text{curve}}(\theta; \tilde{p}_\theta, \tilde{F}_{\mathrm{ref}}) & = \mathbb{E}_{x\sim d_0}\left[ \frac{\tilde{f}_{\text{ref}}\left(\tilde{p}_\theta(x)\right)}{\tilde{F}_\text{ref}\left(\tilde{p}_\theta(x)\right)} \nabla_\theta \tilde{p}_\theta(x)\right] \\
      & = \mathbb{E}_{x\sim d_0}\left[ \frac{f_{\text{ref}}\left(p_\theta(x)\right)}{F_\text{ref}\left(p_\theta(x)\right)} \nabla_\theta p_\theta(x)\right] = \nabla_\theta J_{\text{curve}}\left(\theta; p_\theta, F_{\mathrm{ref}}\right).
\end{align*}
\end{proof}

For the pointwise objective function, this monotone calibration invariance property does not hold. Immediately, we have the following corollary.
\begin{corollary} Denote $ J_g(\theta; p_\theta) = \mathbb{E}_{x \sim d_0}\left[g\left(p_\theta(x)\right)\right]$ and $ J_g(\theta; \tilde{p}_\theta) = \mathbb{E}_{x \sim d_0}\left[g\left(\tilde{p}_\theta(x)\right)\right]$ with $\tilde{p}_\theta = G \circ p_\theta$. There exists some $g$ such that 
    \begin{align*}
   \nabla_\theta J_g(\theta; p_\theta) \neq \nabla_\theta J_g(\theta; \tilde{p}_\theta). 
\end{align*}

\end{corollary}
\begin{proof}
Firstly, we have
    \begin{align*}
        \nabla_\theta J_g\left(\theta ; \tilde{p}_\theta\right)=\mathbb{E}_{x \sim d_0}\left[g^{\prime}\left(\tilde{p}_\theta(x)\right) \nabla_\theta \tilde{p}_\theta(x)\right]=\mathbb{E}_{x \sim d_0}\left[g^{\prime}\left(G\left(p_\theta(x)\right)\right) \cdot G^{\prime}\left(p_\theta(x)\right) \cdot \nabla_\theta p_\theta(x)\right].
    \end{align*}
If we let
\begin{align*}
    \nabla_\theta J_g(\theta)=\mathbb{E}_{x \sim d_0}\left[g^{\prime}\left(p_\theta(x)\right) \nabla_\theta p_\theta(x)\right] = \mathbb{E}_{x \sim d_0}\left[g^{\prime}\left(G\left(p_\theta(x)\right)\right) \cdot G^{\prime}\left(p_\theta(x)\right) \cdot \nabla_\theta p_\theta(x)\right],
\end{align*}
this implies a functional equation for all $p_\theta(x)$ and $G$.
\begin{align*}
    g^{\prime}\left(p_\theta(x)\right) = g^{\prime}\left(G\left(p_\theta(x)\right)\right) G^{\prime}\left(p_\theta(x)\right),
\end{align*}
A counterexample is when $g'(t)=\frac{1}{t}$ in MaxRL and $G(t)=t^2$, we have:
\begin{align*}
    \frac{1}{p_\theta(x)} \neq \frac{1}{p_\theta(x)^2} 2p_\theta(x) = \frac{2}{p_\theta(x)}.
\end{align*}
In summary, our distribution-aware prompt reweighting enjoys the monotone calibration invariance and is thus robust to the miscalibration for the absolute value of pass rates. By contrast, the pointwise prompt reweighting does not satisfy this property in general.
\end{proof}

% \clearpage
% \input{Appendix_ImplementationDetails}
% !TEX root = ./neurips_2026.tex

\section{Implementation Details}
\label{app:impl-details}

\paragraph{More Details of Experimental Setup.} Based on verl framework, our training applies a fully on-policy policy gradient update, meaning that there is no importance ratio or associated clipping. We also disable both the KL penalty and the entropy bonus. As such, the effect of the prompt-weighting function is not entangled with auxiliary regularizers, thereby avoiding complex tuning. This practice is also adopted in prior work on RLVR for LLM reasoning~\citep{tajwar2026maximum,olmo2025olmo,tang2026multiplex,zhang2025rediscovering}. During training across $1000$ RL steps, we employ the temperature $1.0$ and the learning rate $1\times 10^{-6}$. The maximum prompt length is capped at $1024$ tokens and the maximum response length at $4096$ tokens for both Qwen3-1.7B-Base and Qwen3-4B-Base. In evaluation, we use $\mathrm{pass@}k$~\citep{chen2021evaluating} as the primary metric, with $k\in\{1,2,4,\dots,1024\}$ for Qwen3-1.7B-Base and $k\in\{1,2,4,\dots,512\}$ for Qwen3-4B-Base. $\mathrm{pass@}k$ measures the probability that at least one of $k$ independently sampled responses is correct, serving as a proxy for a model's exploration capability under a fixed sampling budget. In practice, $\mathrm{pass@}1$ is the raw mean
accuracy, while $\mathrm{pass@}k$ for $k\geq 2$ uses $1000$ best-of-$k$ bootstrap resamples (with replacement) of size $k$, averaged across prompts. We sample $2048$ rollouts per prompt for Qwen3-1.7B-Base and $1024$ for Qwen3-4B-Base to reduce estimation variance.

\paragraph{RL Stack and Computational Devices.} All experiments are run within the verl post-training framework~\citep{sheng2025hybridflow}, which couples a policy-gradient trainer with a vLLM-backed rollout engine. Using a single inference stack for both training rollouts and downstream evaluation removes a common source of distribution shift between the two stages. Both backbones are optimized with the AdamW optimizer using its default first- and second-moment coefficients, under bfloat16 mixed precision and FSDP-style sharding of parameters, gradients, and optimizer states. Each experiment runs on a single node of $8\times$ NVIDIA B200 GPUs interconnected by NVLink.

\begin{tcolorbox}[
    breakable, enhanced, colback=gray!5, colframe=black!60,
    title=\textbf{Qwen-math prompt template}, fonttitle=\bfseries,
]
\begin{verbatim}
<|im_start|>system
Please reason step by step and put the final answer in \boxed{}.
<|im_end|>
<|im_start|>user
{problem statement}
Let's think step by step and put the final answer within \boxed{}.
<|im_end|>
<|im_start|>assistant
\end{verbatim}
\end{tcolorbox}

\paragraph{Prompt Template.} Every training and evaluation prompt is rendered with the Qwen-math chat template~\citep{yang2025qwen3}. We append the instruction ``\texttt{Let's think step by step and put the final answer within \textbackslash boxed\{\}.}'' to the problem statement and route the result through the tokenizer's chat template, yielding inputs of the form shown below. At inference time the verifier extracts the final boxed expression from the generated chain of thought and grades it with \textsc{Math-Verify}.

\paragraph{On-policy Update.} Each RL step generates a fresh batch of rollouts from the current policy and consumes it with a single optimizer update. Therefore, the data-generating policy and the target policy coincide. The token-level importance ratio is identically one along the entire response,
\[
    \rho_{i,t}(\theta) \;=\; \frac{\pi_\theta(y_{i,t}\mid x,\,y_{i,<t})}
                                {\pi_{\theta_{\text{old}}}(y_{i,t}\mid x,\,y_{i,<t})}
    \;\equiv\; 1,
\]
which implies that the PPO clipping range $[1-\epsilon,\,1+\epsilon]$ is never active and the threshold $\epsilon$ has no effect on the optimization trajectory. The KL penalty and the entropy bonus are also disabled, isolating the contribution of the prompt-weighting rule from auxiliary regularizers and matching the simplified objective adopted by \citep{tajwar2026maximum}. This supports the concise formulations in the preliminaries of \Cref{sec:preliminaries}.

\paragraph{Evaluation Decoding.} Following \citep{tajwar2026maximum}, evaluation generations are sampled at temperature $0.6$ with top-$p$ $0.95$, and both top-$k$ and min-$p$ truncation are disabled. We do not employ adaptive sampling and do not apply any logit correction for the mismatch between inference and training runtimes. Reported metrics are computed on the final checkpoint of each run.

\begin{table}[htbp]
\centering
\caption{Training hyperparameters shared by both Qwen3 backbones.}
\label{tab:train-hparams}
\small
\begin{tabular}{ll@{\hskip 2em}ll}
\toprule
Parameter & Value & Parameter & Value \\
\midrule
Base model            & Qwen3-\{1.7B, 4B\}-Base & Prompts per batch     & $256$ \\
Rollouts per prompt   & $8$                     & Grad updates per step & $1$ \\
Max prompt length     & $1024$                  & Max response length   & $4096$ \\
Learning rate         & $1\times 10^{-6}$       & Training steps        & $1000$ \\
KL coefficient        & $0$                     & Entropy coefficient   & $0$ \\
Rollout temperature   & $1.0$                   & Eval temperature      & $0.6$ \\
Eval top-$p$          & $0.95$                  & Device                & $8 \times$ NVIDIA B200 \\
\bottomrule
\end{tabular}
\end{table}

\paragraph{Hyperparameter Summary.} \Cref{tab:train-hparams} consolidates the training-side hyperparameters, which are shared across both Qwen3-1.7B-Base and Qwen3-4B-Base.

\clearpage
% \input{Appendix_SupplementalResults}
% !TEX root = ./neurips_2026.tex

\section{Supplemental Experimental Results}
\label{app:supp-results}

This appendix contains experimental results that are not shown in the main text due to the space limit. The training setup, evaluation protocol, and bootstrap procedure are identical to those described in \Cref{sec:exp-setup}.

\subsection{Pass@1 and Pass@$64$ on Three Additional Benchmarks}
\label{app:supp-table}

\begin{table}[htbp]
\centering
\caption{\textbf{Supplemental results on three additional math reasoning benchmarks.} pass@1 and pass@64 (\%), best per column in bold.}
\label{tab:supp-results}
\setlength{\tabcolsep}{3pt}
\small
% \resizebox{\textwidth}{!}{%
\scalebox{0.8}{
\begin{tabular}{lcccccccc}
\toprule
\multicolumn{1}{c}{\textbf{Method}} & \multicolumn{2}{c}{\textbf{BRUMO 2025}} & \multicolumn{2}{c}{\textbf{HMMT 11/25}} & \multicolumn{2}{c}{\textbf{Minerva}} & \multicolumn{2}{c}{\textbf{Avg.}} \\
\cmidrule(lr){2-3} \cmidrule(lr){4-5} \cmidrule(lr){6-7} \cmidrule(lr){8-9}
 & pass@1 & pass@64 & pass@1 & pass@64 & pass@1 & pass@64 & pass@1 & pass@64 \\
\midrule
\multicolumn{9}{c}{\emph{Qwen3-1.7B-Base}} \\
\midrule
Base & 8.0 & 40.1 & 2.4 & 18.7 & 20.8 & 56.9 & 10.4 & 38.6 \\
GRPO & 14.2 & 40.2 & 3.0 & 14.7 & \textbf{31.4} & 57.7 & 16.2 & 37.5 \\
MaxRL & 15.9 & 42.5 & 3.7 & 18.4 & 31.2 & 59.5 & 16.9 & 40.2 \\
\rowcolor{gray!18} \textbf{CurveRL} & \textbf{18.0} & \textbf{45.3} & \textbf{3.9} & \textbf{20.6} & 30.3 & \textbf{60.3} & \textbf{17.4} & \textbf{42.1} \\
\midrule
\multicolumn{9}{c}{\emph{Qwen3-4B-Base}} \\
\midrule
Base & 18.2 & 46.7 & 3.5 & 21.3 & 25.9 & 61.3 & 15.9 & 43.1 \\
GRPO & 27.3 & 47.1 & 5.7 & 27.8 & 40.8 & 58.3 & 24.6 & 44.4 \\
MaxRL & \textbf{30.4} & \textbf{66.5} & \textbf{8.1} & 35.6 & \textbf{42.8} & 62.3 & \textbf{27.1} & 54.8 \\
\rowcolor{gray!18} \textbf{CurveRL} & 27.6 & 66.2 & 7.2 & \textbf{39.1} & 39.9 & \textbf{62.9} & 24.9 & \textbf{56.1} \\
\bottomrule
\end{tabular}%

}
% }
\end{table}

\Cref{tab:supp-results} reports pass@$1$ and pass@$64$ on three additional
benchmarks: \textbf{BRUMO 2025}~\citep{balunovic2025matharena},
\textbf{HMMT 11/25}~\citep{balunovic2025matharena}, and
\textbf{Minerva Math}~\citep{lewkowycz2022solving}. The qualitative picture
mirrors~\Cref{tab:main-results}: CurveRL matches or exceeds MaxRL at
pass@$64$ on all benchmarks across both model sizes (except BRUMO 2025 under Qwen3-4B-Base).

% \subsection{Majority Voting Accuracy at the Maximum Sample Budget}
% \label{app:majmaxk}

% \Cref{tab:majmaxk-4b} reports majority-vote accuracy at the maximum
% per-prompt rollout budget on Qwen3-4B-Base (majority@1024). CurveRL is
% column-best on BeyondAIME, HMMT 02/26, and HMMT 11/25, and ties for first
% on HMMT 02/25; MaxRL retains the lead on AIME 2025, BRUMO 2025, and
% Minerva.

% \begin{table}[htbp]
% \centering
% \caption{\textbf{Majority@1024 on Qwen3-4B-Base} across the eight evaluation benchmarks. Best per column in bold.}
% \label{tab:majmaxk-4b}
% \setlength{\tabcolsep}{3pt}
% \small
% \resizebox{\textwidth}{!}{%
% \begin{tabular}{lccccccccc}
% \toprule
% \textbf{Method} & \textbf{AIME 2025} & \textbf{BeyondAIME} & \textbf{HMMT 02/25} & \textbf{HMMT 02/26} & \textbf{MATH-500} & \textbf{BRUMO 2025} & \textbf{HMMT 11/25} & \textbf{Minerva} & \textbf{Avg.} \\
% \midrule
% Base & 14.8 & 7.0 & 0.0 & 6.1 & 81.8 & 23.9 & 3.3 & 32.0 & 21.1 \\
% GRPO & \textbf{24.3} & 11.0 & \textbf{10.2} & 14.0 & 86.0 & 31.0 & 3.5 & 44.1 & 28.0 \\
% MaxRL & 23.5 & 11.1 & 10.0 & 12.1 & \textbf{89.9} & \textbf{38.6} & 10.0 & \textbf{48.5} & \textbf{30.5} \\
% \rowcolor{gray!18} \textbf{CurveRL} & 23.4 & \textbf{13.0} & \textbf{10.2} & \textbf{18.3} & 89.3 & 33.5 & \textbf{10.1} & 45.0 & 30.3 \\
% \bottomrule
% \end{tabular}%
% }
% \end{table}

% \clearpage
\subsection{Pass@$k$ Curves on the Additional Benchmarks}
\label{app:supp-figure}

\begin{figure}[htbp]
    \centering
    \includegraphics[width=0.7\linewidth]{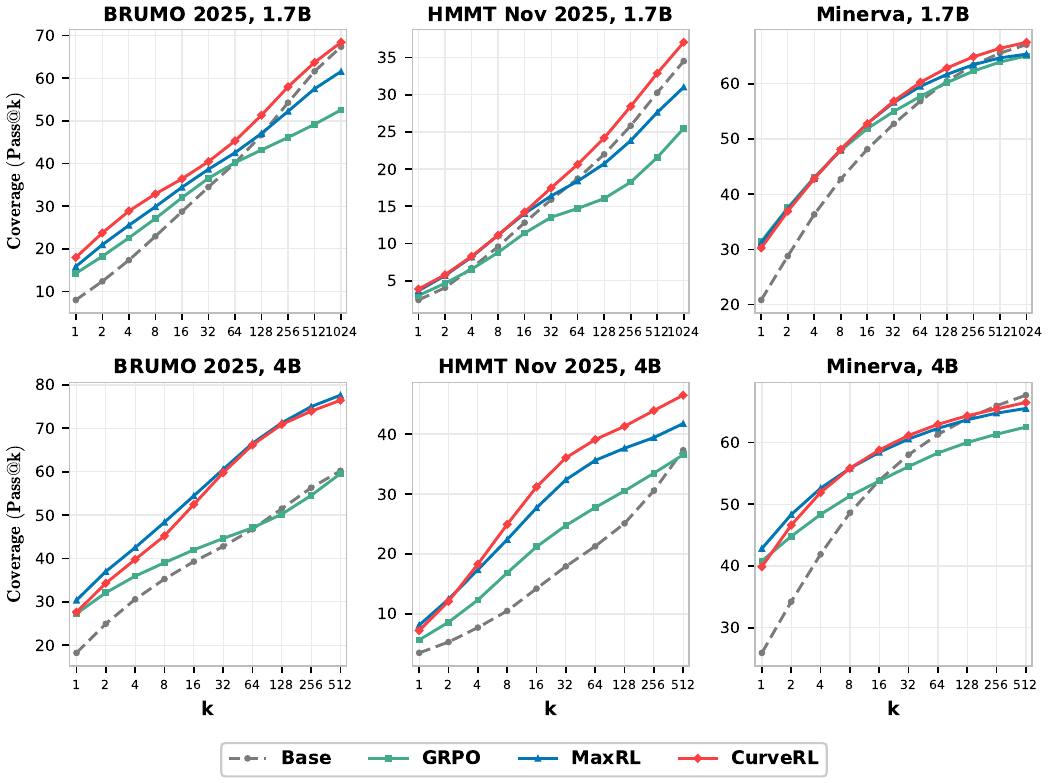}
    \caption{\textbf{Pass@$k$ scaling on the additional three benchmarks.}
    Top row: Qwen3-1.7B-Base, $k\in\{1,\dots,1024\}$. Bottom row:
    Qwen3-4B-Base, $k\in\{1,\dots,512\}$. CurveRL tracks or dominates the
    GRPO and MaxRL baselines on most benchmarks.}
    \label{fig:appendix-passk}
\end{figure}

\clearpage
\subsection{Additional Difficulty Distribution}
\label{app:difficulty-supp}

\begin{figure}[h]
    \centering
    \includegraphics[width=0.8\linewidth]{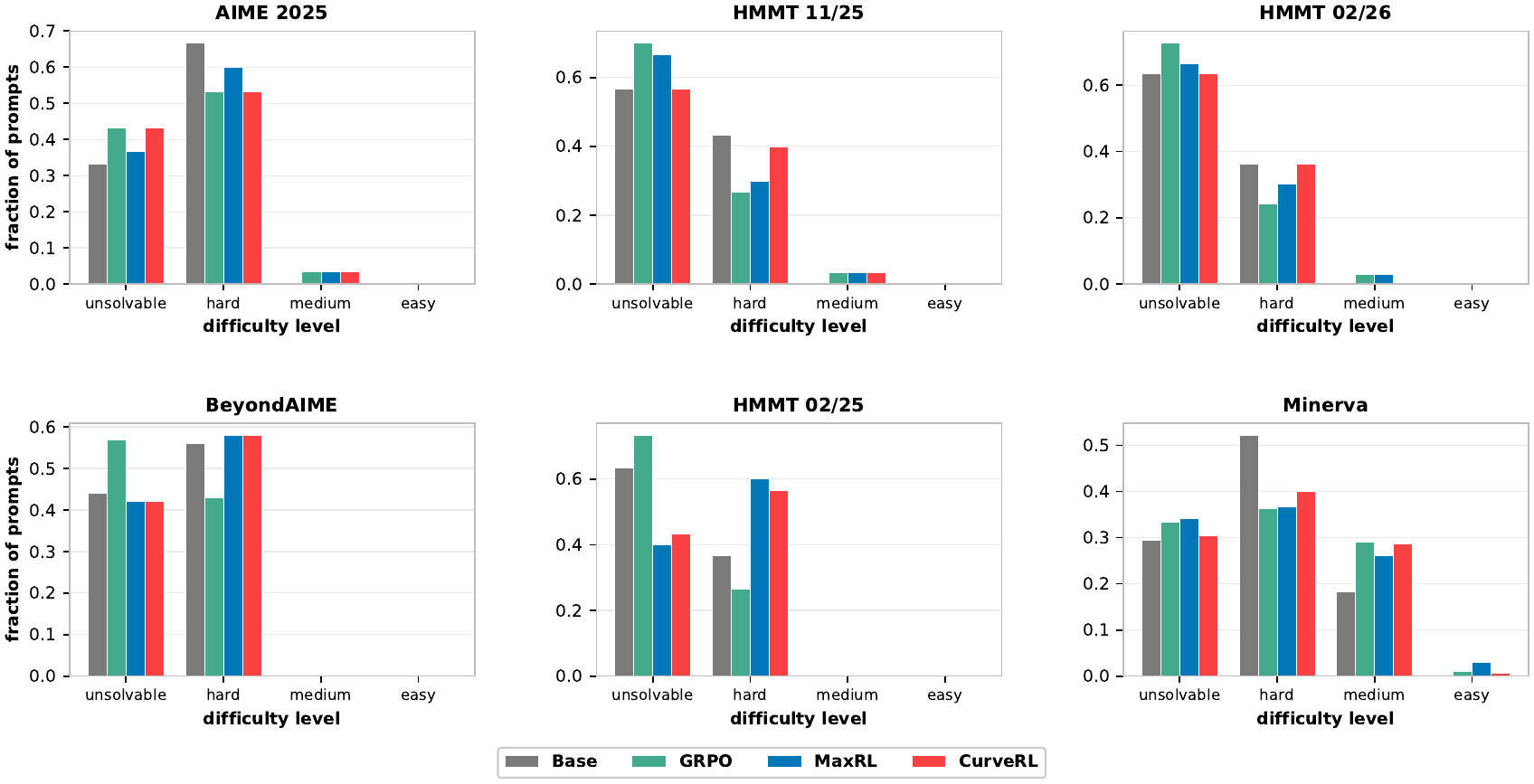}
    \caption{\textbf{Qwen3-1.7B-Base post-training prompt-difficulty distribution on the remaining six benchmarks.} Across benchmarks, CurveRL's \emph{unsolvable} fraction is mostly no higher than that of the strongest pointwise-weighted baseline, while the \emph{hard}/\emph{medium} mass is preserved or enlarged.}
    \label{fig:difficulty-appendix}
\end{figure}

\begin{figure}[h]
    \centering
    \includegraphics[width=\linewidth]{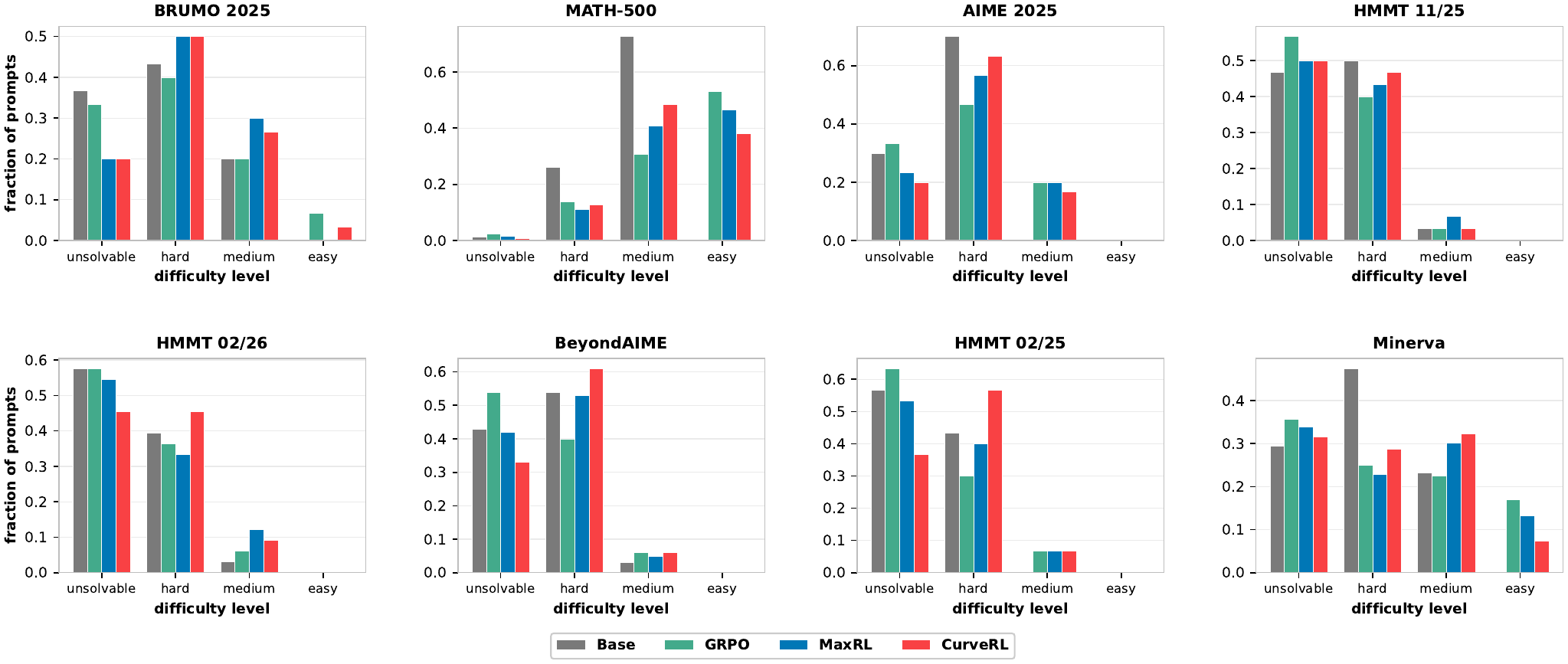}
    \caption{\textbf{Qwen3-4B-Base post-training prompt-difficulty distribution on all eight benchmarks.} CurveRL's \emph{unsolvable} fraction tracks or improves on the strongest pointwise-weighted baseline at the 4B scale as well, mirroring the 1.7B trend in \Cref{fig:difficulty-appendix}.}
    \label{fig:difficulty-combined-4b}
\end{figure}

\clearpage
\subsection{Learning Signals Across Training Dynamics}
\label{app:nonzero-pass-rate}

\Cref{fig:nonzero-pass-rate} tracks the fraction of training prompts with strictly positive empirical pass rate 
$\hat p$ at each RL step on Qwen3-1.7B-Base and Qwen3-4B-Base. Larger values indicate a richer pool of prompts that 
provide non-vanishing gradient signals. While all three methods start from a similar pool, CurveRL and MaxRL remain consistently above GRPO throughout the training, with the gap widening after the early phase. This suggests that they slow down the collapse of prompts into the zero-gradient $\hat p=0$ region, preserving useful learning signals for more steps and 
broadening the reasoning boundary.

\begin{figure}[htbp]
    \centering
    \includegraphics[width=0.8\linewidth]{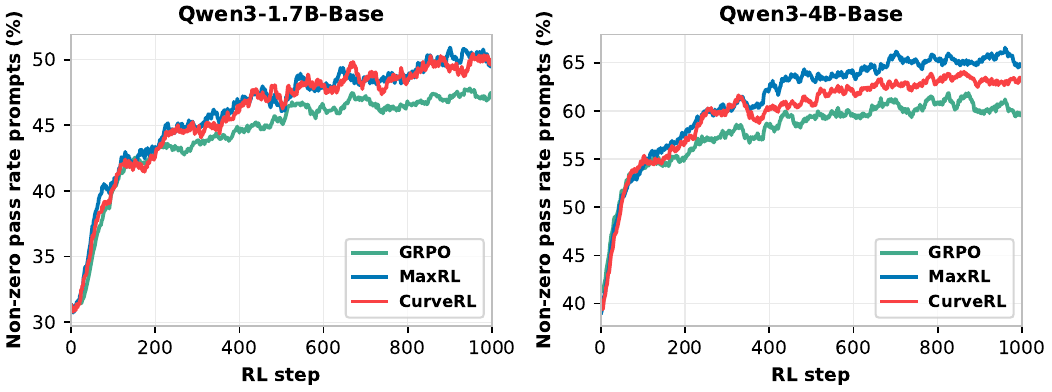}
    \caption{Fraction of prompts where the model generates at least one correct rollout out of 8 samples.}
    \label{fig:nonzero-pass-rate}
\end{figure}

% \clearpage
\subsection{Supplemental Results on Sensitivity Analysis}
\label{app:sensitivity-t0}

% \begin{table}[h]
% \centering
% \caption{\textbf{Sensitivity to sliding-window size $t_0$ on the three additional benchmarks.} pass@1 / pass@64 (\%) on Qwen3-1.7B-Base; best per column in bold; ``$-$'' marks pending runs.}
% \label{tab:sensitivity-t0-supp}
% \setlength{\tabcolsep}{4pt}
% \small
% \begin{tabular}{lcccccccc}
% \toprule
% \multicolumn{1}{c}{\textbf{$t_0$}} & \multicolumn{2}{c}{\textbf{BeyondAIME}} & \multicolumn{2}{c}{\textbf{HMMT 02/25}} & \multicolumn{2}{c}{\textbf{Minerva}} & \multicolumn{2}{c}{\textbf{Avg.}} \\
% \cmidrule(lr){2-3} \cmidrule(lr){4-5} \cmidrule(lr){6-7} \cmidrule(lr){8-9}
%  & pass@1 & pass@64 & pass@1 & pass@64 & pass@1 & pass@64 & pass@1 & pass@64 \\
% \midrule
% $1$ & \textbf{1.4} & 21.6 & 1.0 & 17.3 & \textbf{30.4} & 58.7 & 10.9 & 32.5 \\
% \rowcolor{gray!18} $10$ (default) & \textbf{1.4} & \textbf{23.6} & \textbf{1.7} & \textbf{19.3} & 30.3 & \textbf{60.3} & \textbf{11.1} & \textbf{34.4} \\
% $50$ & $-$ & $-$ & $-$ & $-$ & $-$ & $-$ & $-$ & $-$ \\
% \bottomrule
% \end{tabular}
% \end{table}

% \Cref{tab:sensitivity-t0-supp} mirrors \Cref{tab:sensitivity-t0} on BeyondAIME, HMMT 02/25, and Minerva. The qualitative pattern carries over: the default $t_0\!=\!10$ matches or improves on the single-batch setting at every column, with the average pass@$64$ advantage (\textbf{$+1.9$\,pp}) again concentrated at moderate-to-large $k$. The $t_0\!=\!50$ row is left as placeholder pending the corresponding evaluation runs.

\begin{figure}[htbp]
    \centering
    \includegraphics[width=0.8\linewidth]{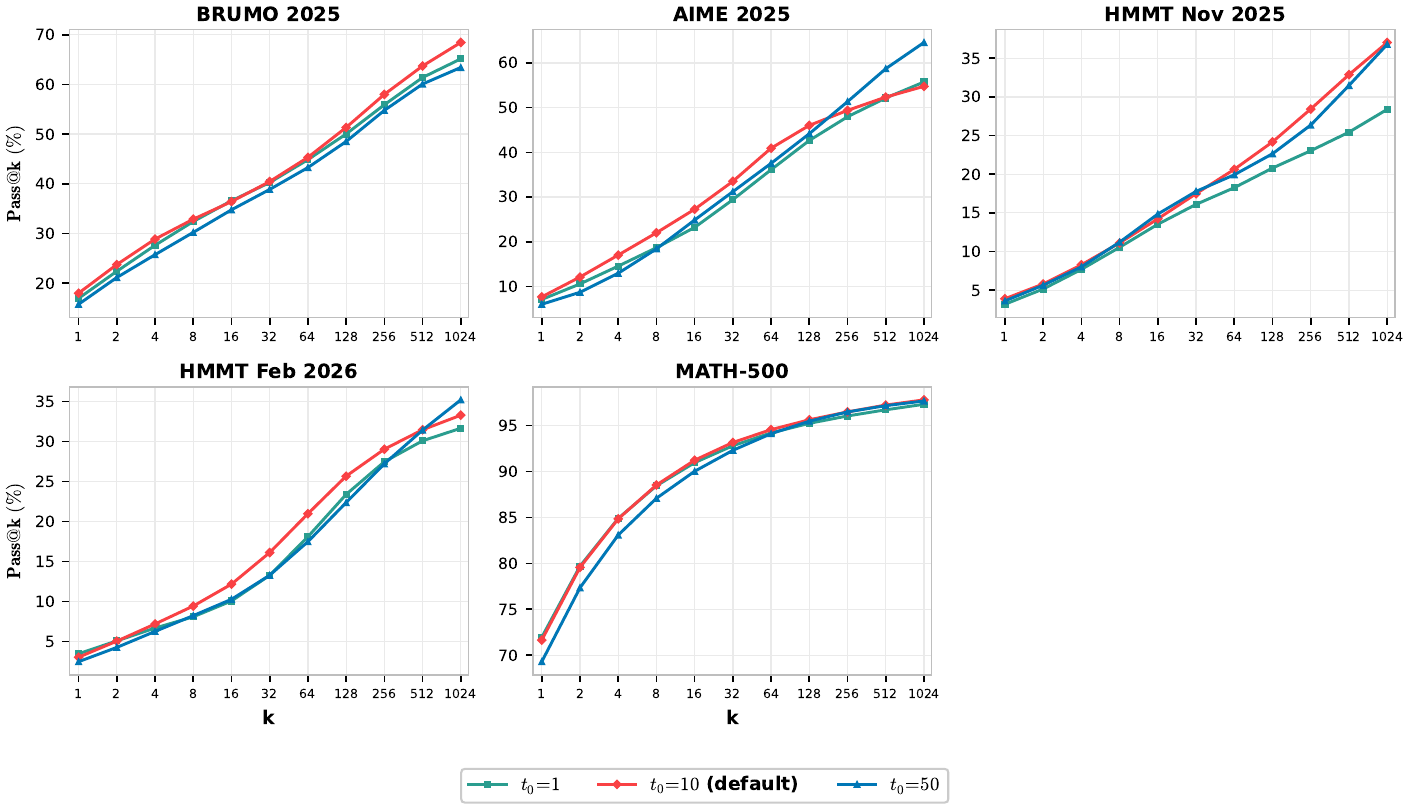}
    \caption{\textbf{Pass@$k$ scaling for the three sliding-window sizes.} Five benchmarks from \Cref{tab:sensitivity-t0}; each panel overlays CurveRL with $t_0\!\in\!\{1,10,50\}$ batches on Qwen3-1.7B-Base. The default $t_0\!=\!10$ tracks or improves on both alternatives at every $k$ on the most benchmarks.}
    \label{fig:sensitivity-passk}
\end{figure}

\clearpage
\subsection{Stable Training Dynamics}
\label{app:training-dynamics}

\Cref{fig:training-dynamics} reports additional training dynamics on Qwen3-1.7B-Base (top row) and Qwen3-4B-Base
(bottom row), including mean response length, policy entropy, and gradient norm. CurveRL generally produces longer
chains of thought and maintains higher actor entropy, supporting greater reasoning diversity and a wider capability
boundary, as shown in \Cref{sec:experiment}. CurveRL also exhibits a flatter gradient-norm profile, indicating more
stable training.

\begin{figure}[h]
    \centering
    \includegraphics[width=\linewidth]{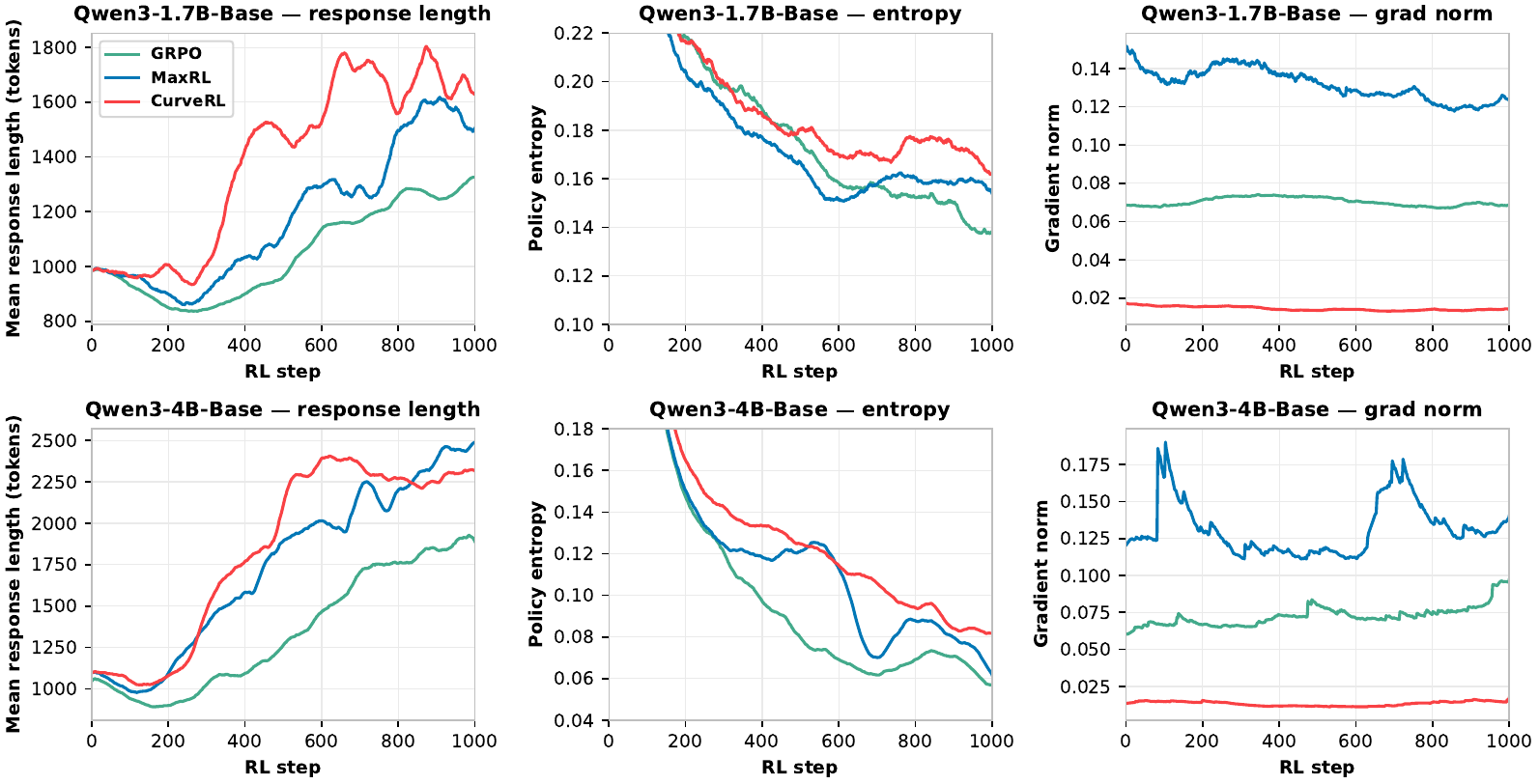}
    \caption{\textbf{Additional training dynamics metrics for Qwen3-1.7B-Base (top row) and Qwen3-4B-Base (bottom row).} Mean
    response length, policy entropy, and gradient norm over RL steps for three RLVR algorithms.}
    \label{fig:training-dynamics}
\end{figure}

\clearpage
\subsection{Validation Accuracy During Training}
\label{app:validation-curves}

\Cref{fig:validation-curves,fig:validation-curves-4b} report the validation accuracy on the MATH-500 during training. CurveRL generally outperforms other baselines on both pass@$1$ and
pass@$k$ during training.

\begin{figure}[h]
    \centering
    \includegraphics[width=0.7\linewidth]{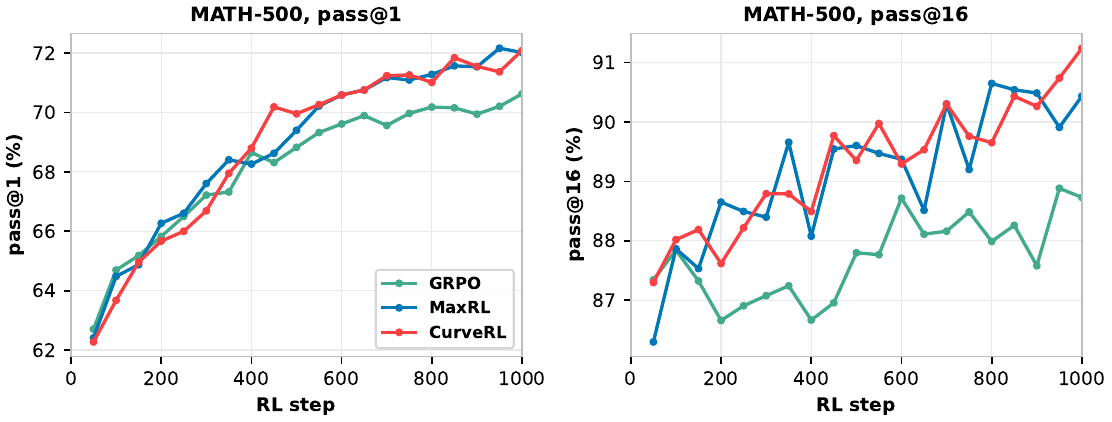}
    \caption{\textbf{Qwen3-1.7B-Base validation accuracy during training.}}
    \label{fig:validation-curves}
\end{figure}

\begin{figure}[h]
    \centering
    \includegraphics[width=0.7\linewidth]{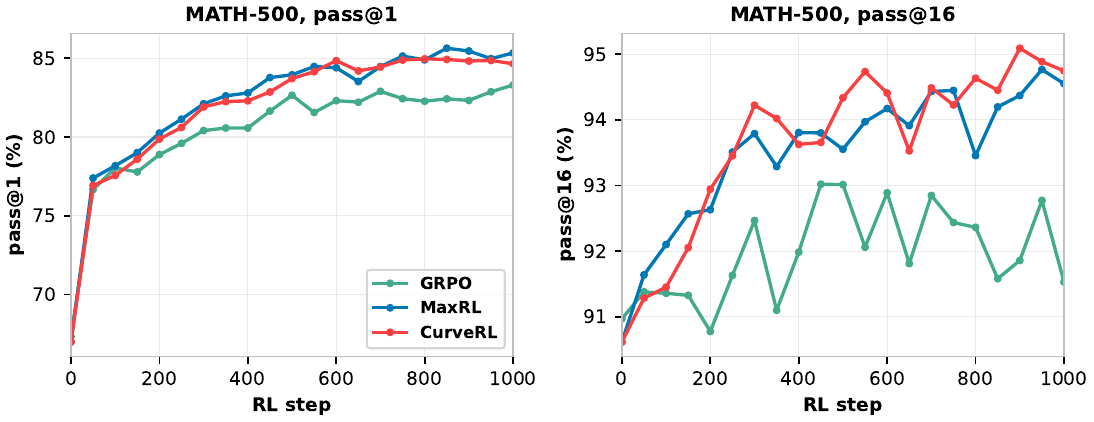}
    \caption{\textbf{Qwen3-4B-Base validation accuracy during training.}}
    \label{fig:validation-curves-4b}
\end{figure}

\subsection{Distribution-Aware Weighting on Qwen3-1.7B-Base}
\label{app:weight-comparison-4b}

Results shown in \Cref{fig:weight-comparison-4b} on Qwen3-1.7B-Base are similar to \Cref{fig:weight-comparison} on Qwen3-4B-Base. The qualitative pattern
matches the 4B case in the main text: the empirical pass-rate density $\hat f_{\mathrm{ref}}$ drifts toward higher
pass-rate bins as training proceeds, while the data-driven weight $w_t = \hat f_{\mathrm{ref}}/\hat F_{\mathrm{ref}}$
continues to emphasize the low-pass-rate prompts. The static GRPO and MaxRL weights are model-size invariant and therefore identical to the 4B panels.

\begin{figure}[h]
    \centering
    \includegraphics[width=\linewidth]{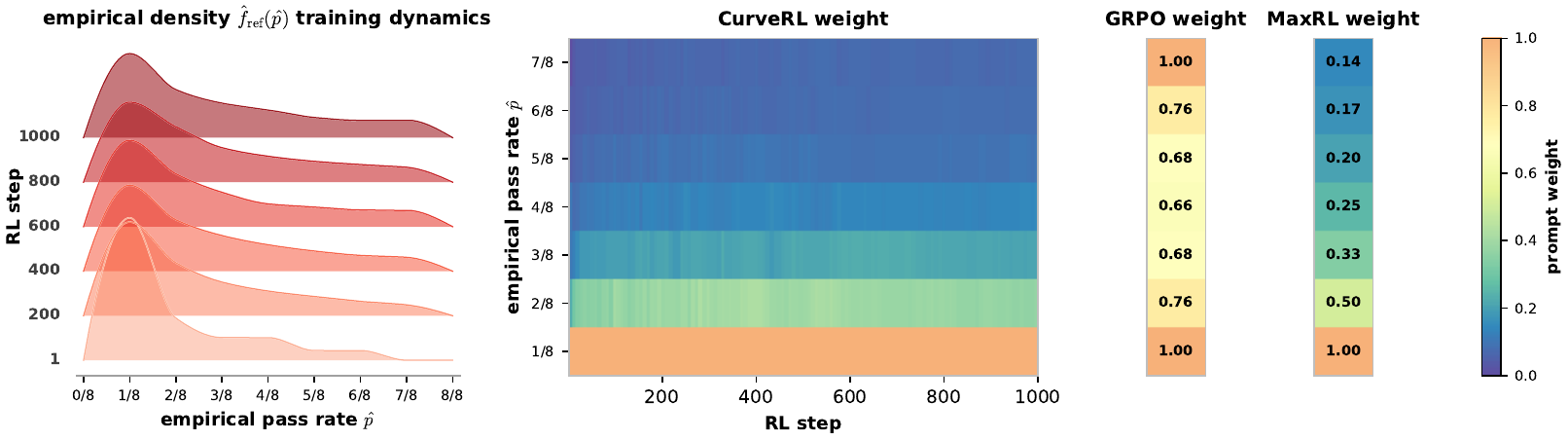}
    \caption{\textbf{Weighting schemes on Qwen3-1.7B-Base.} Same layout and scale as \Cref{fig:weight-comparison}.
    Only CurveRL's panels (left two) reflect 1.7B training, since the GRPO and MaxRL weights are static.}
    \label{fig:weight-comparison-4b}
\end{figure}

\clearpage
\section{More Discussions}\label{app:discussion}

\subsection{Curriculum Learning under Our Framework}\label{app:discussion_curriculum}

\paragraph{Curriculum Learning is a Time-Varying Pointwise Prompt Reweighting Method.} The curriculum learning strategy in RLVR, such as \citep{parashar2025curriculum}, develops different data scheduling from easy to hard prompts, therefore improving the policy improvement. Within our framework in Definition~\ref{def:weight}, it indeed changes the measure from a fixed $d_0$ to a time-varying measure $d_0^{(t)}$ that is manually developed, which is orthogonal to the policy-dependent prompt reweighting $w_\theta(x)$. The utility function $\mathcal{U}_\theta$ based on the pointwise objective $J_g(\theta)$ in Eq.~\eqref{eq:objective_pointwise} is changed to a time-varying version:
\begin{align*}
     \mathcal{U}_\theta^{(t)} = \mathbb{E}_{x \sim d_0^{(t)}}\left[g\left(p_\theta(x)\right)\right] = \mathbb{E}_{x \sim d_0}\left[ \frac{d_0^{(t)}(x)}{d_0(x)} g\left(p_\theta(x)\right)\right].
\end{align*}
By the functional derivative defined in Definition~\ref{def:weight}, for any perturbation 
$h \in L^2(d_0^{(t)})$, according to the Riesz representation, the first variation of $\mathcal{U}^{(t)}_\theta$ is given by
\begin{align*}
\lim_{\epsilon \to 0}
\frac{\mathcal U^{(t)}_\theta(p_\theta + \epsilon h) - \mathcal U^{(t)}_\theta(p_\theta)}{\epsilon}
=
\int w^*_\theta(x)\, h(x)\, d^{(t)}_0(x) dx  = \int  w^*_\theta(x)\, \frac{d^{(t)}_0(x)}{d_0(x)} \, h(x)\, d_0(x) dx.
\end{align*}
where $\frac{d_0^{(t)}(x)}{d_0(x)}$ is the  Radon-Nikodym derivative. Immediately, we can derive a time-varying optimal weight function for the pointwise objective $J_g(\theta)$ as
\begin{align*}
    w_\theta^{*}(x, t) = \frac{\delta \mathcal{U}_\theta^{(t)}}{\delta p_\theta}(x) = \frac{d_0^{(t)}(x)}{d_0(x)}  g'(p_\theta(x)).
\end{align*}
This illuminates the role of curriculum learning in RLVR: it incorporates a pointwise time-varying prompt reweighting by manually developing $d_0^{(t)}$ in a curriculum way, which is on top of the policy-dependent prompt reweighting function $w_\theta(x)$.

\paragraph{Similarity and Difference Between Curriculum Learning and Our Distribution-aware Method.} Given the aforementioned mechanism of curriculum learning in RLVR, a natural question is: \textit{why curriculum learning and our distribution-aware method are both effective in RLVR and what the fundamental distinction is}. We argue that both methods aim to reshape the learning dynamics of RLVR to focus on prompts near the learning frontier or in the current learnable window at each training time by prompt reweighting, even though they are in different ways:
\begin{itemize}[leftmargin=*]
    \item The easy-to-hard schedule employed in curriculum learning incorporates a non-stationary sample allocation mechanism to help the policy optimization to \textbf{chase} the moving learnable window. However, the endogenous easy-to-hard schedule is typically conducted by leveraging prior knowledge and based on the absolute values of pass rates, to the best of our knowledge.
    \item By introducing the quantile coordinate transform, i.e., quantile reparameterization, our distribution-aware method flattens the difficulty distribution in a nearly uniform and stationary spectrum, transforming the moving learnable windows to a stationary and fixed one. This strategy implicitly yet adaptively helps the policy to concentrate on the prompts in the learning frontier, and at the same time it takes advantage of the distributional information of pass rates in the pass rate function space.
\end{itemize}

\subsection{When Gradients of Distribution-aware and Pointwise Utility Functionals Equal}\label{app:discussion_equivalence}
\begin{proof}
    Under the same distortion function, the distribution-aware and pointwise utility functions are
\begin{align*}
    \mathcal{U}_\theta(F_{\mathrm{ref}})=\mathbb{E}_{x \sim d_0}\left[\psi\left(F_{\text {ref }}\left(p_\theta(x)\right)\right)\right], \ \mathcal{U}_\theta:=J_\psi(\theta)=\mathbb{E}_{x \sim d_0}\left[\psi\left(p_\theta(x)\right)\right],
\end{align*}
with their respective gradients:
\begin{align*}
    \nabla_\theta \mathcal{U}_\theta(F_{\mathrm{ref}})&=\mathbb{E}_{x \sim d_0}\left[\psi^{\prime}\left(F_{\text {ref }}\left(p_\theta(x)\right)\right) f_{\text {ref }}\left(p_\theta(x)\right) \nabla_\theta p_\theta(x)\right], \\
    \nabla_\theta J_\psi(\theta) &=\mathbb{E}_{x \sim d_0}\left[\psi^{\prime}\left(p_\theta(x)\right) \nabla_\theta p_\theta(x)\right].
\end{align*}
Denote $p=p_\theta(x)$. For the general increasing function $\psi$, the equivalence of the two gradients implies
\begin{align*}
    \psi^\prime(p) = \psi^\prime({F_{\text {ref }}(p)}) f_{\text {ref }}(p) \Rightarrow \frac{d}{d p} \psi\left(F_{\mathrm{ref}}(p)\right)=\frac{d}{d p} \psi(p).
\end{align*}
By taking the integral on both sides between $p$ and 1, we have
\begin{align*}
    \psi(F_{\mathrm{ref}}(1)) - \psi(F_{\mathrm{ref}}(p)) = \psi(1) - \psi(p).
\end{align*}
As we know $F_{\mathrm{ref}}(1)=1$, this implies that
\begin{align*}
    \psi(F_{\mathrm{ref}}(p)) = \psi(p),
\end{align*}
which holds for each $p$. As $\psi$ is an increasing function, we can derive
\begin{align*}
   F_{\mathrm{ref}}(p) = p, 
\end{align*}
which implies that $F_{\mathrm{ref}}$ is $\text{Uniform}(0, 1)$ for the pass rate distributions of $p_\theta(x)$. 

On the other proof direction, we assume $F_{\mathrm{ref}}$ is $\text{Uniform}(0, 1)$, which immediately implies $\psi(F_{\mathrm{ref}}(p)) = \psi(p)$. Consequently, the distribution-aware utility function degrades to the pointwise one and this degradation also happens to the prompt weighting function.

\paragraph{Remark.} This connection with the uniform distribution reveals a fundamental source of differences between distribution-aware prompt reweighting and pointwise counterpart: the mismatch between model capability and data difficulty. Specifically, in the general learning dynamics, the pass rate distributions can be more concentrated around 0 if the model struggles to achieve high pass rates for most prompts, especially in the early training phase. That's when the \textit{weight collapse} issue happens and distribution-aware prompt reweighting differs from the pointwise counterpart. For example, if we increase the model size on a fixed training dataset, the advantage may diminish as the pass rate distribution becomes more spread and closer to a uniform distribution. However, this does not happen in the general learning dynamics. Therefore, the fundamental source of the advantage of distribution-aware prompt reweighting over the pointwise version depends on capability-difficulty match between the model and the dataset.
\end{proof}

\subsection{Sufficient Conditions for Prompt Weight Comparison and Risk Preference}\label{app:discussion_weightcomparison}

Under the same risk measure $\psi$, we define the pointwise and distribution-aware prompt weight by
\begin{align*}
    w_0(p)= \psi^\prime(p), \  w_F(p)=\psi^\prime({F_{\text {ref }}(p)}) f_{\text {ref }}(p).
\end{align*}
Following the analysis in Appendix~\ref{app:discussion_equivalence}, we hope to know when $w_0(p) > w_F(p)$ for the concerned pass rates, e.g., low pass rate in the early training. Answering this question recovers a fundamental relationship between model capability and data difficulty. 

\begin{figure}[b!]
    \centering
    \includegraphics[width=1.0\linewidth]{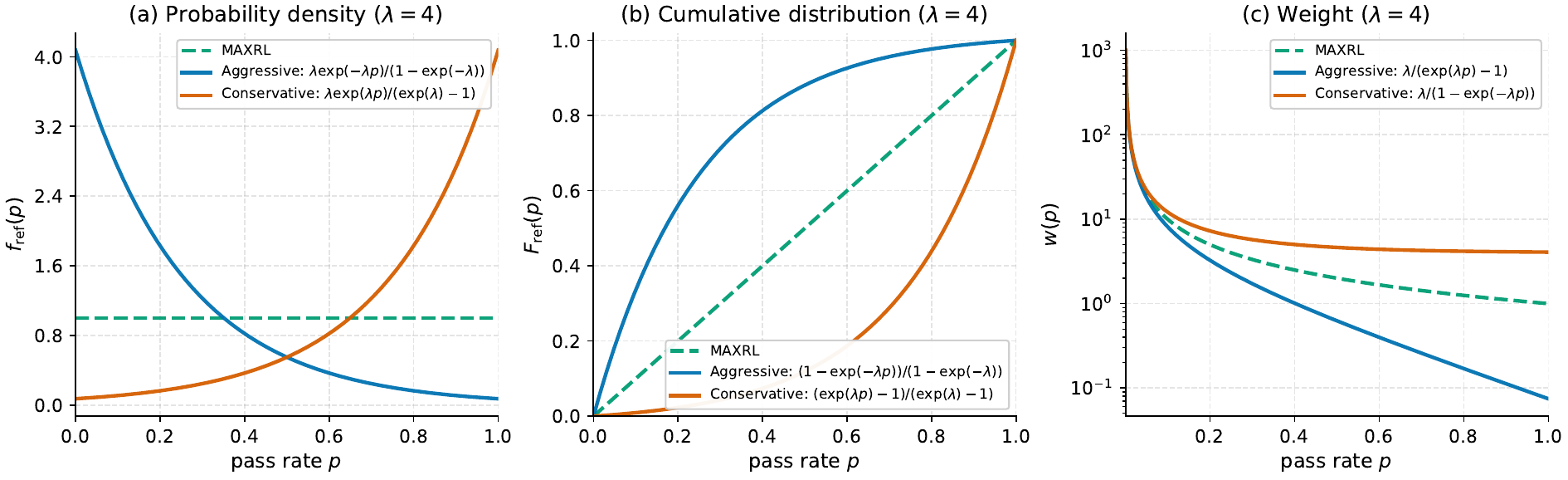}
    \caption{The probability density function, CDF and weight function in terms of the pass rate $p$ among MaxRL and the two considered examples. The blue line shows a more aggressive weight update on the low-pass-rate prompts, while the red line suggests a more conservative update relative to MaxRL.}
    \label{fig:app_weightcomparision}
\end{figure}

\paragraph{For $\psi(t)=\log(t)$.} $w_0(p)=\frac{1}{p}$ and $w_F(p)=\frac{ f_{\mathrm{ref}}(p)}{F_{\mathrm{ref}}(p)}$. As the two weights have different magnitudes, comparing the weight requires us to perform normalization. For $N$ number of rollouts, the pass rate satisfies $p\in [0, \frac{1}{N}, \ldots, \frac{N}{N}]$. Denote $p_i = \frac{i}{N}$ and thus the normalized weights are
\begin{align*}
    \bar{w}_0(p_i) = \frac{w_0(p_i)}{\sum_{i=0}^N w_0(p_i)}, \ \bar{w}_F(p_i) = \frac{w_F(p_i)}{\sum_{i=0}^N w_F(p_i)}.
\end{align*}
We now analyze the scenario when the distribution-aware weight $\bar{w}_F$ is more aggressive than the pointwise weight $\bar{w}_0$. We first define:
\begin{align*}
    r(p)=\frac{w_F(p)}{w_0(p)}=\frac{p f_{\mathrm{ref}}(p)}{F_{\mathrm{ref}}(p)}=\frac{d \log F_{\mathrm{ref}}(p)}{d \log p} = \frac{ f_{\mathrm{ref}}(p)}{\frac{1}{p}\int_0^p f_{\mathrm{ref}}(t) dt}.
\end{align*}
Given any $p_l < p_u$, being more aggressive for $\bar{w}_F$ indicates that
\begin{align*}
    \frac{\bar{w}_F(p_l)}{\bar{w}_F(p_u)} > \frac{\bar{w}_0(p_l)}{\bar{w}_0(p_u)} \Leftrightarrow \frac{{w}_F(p_l)}{{w}_F(p_u)} > \frac{{w}_0(p_l)}{{w}_0(p_u)} \Leftrightarrow r(p_l) > r(p_u).
\end{align*}
Therefore, \textbf{the sufficient condition} is 
\begin{align*}
    r^\prime(p) < 0,
\end{align*}
in the considered range of $p$, which can guarantee $r(p_l) > r(p_u)$ if $p_l < p_u$. Let $r^\prime(p)=0$, we have:
\begin{align*}
    \frac{d \log F_{\mathrm{ref}}(p)}{d \log p} = \alpha \Rightarrow \log F_{\mathrm{ref}}(p) = \alpha \log p + C,
\end{align*}
where $\alpha$ is a constant. Let $F_{\mathrm{ref}}(1)=1$, we have $C=0$ and 
\begin{align*}
    F_{\mathrm{ref}}(p) = p^\alpha, \ f_{\mathrm{ref}}(p) = \alpha p^{\alpha-1}, \ w_F(p)=\frac{\alpha}{p}.
\end{align*}
which is a power function. In our case, $\alpha=1$. This result indicates that if $r(p)$ is a  decreasing function regarding $p$ or $w_F(p)$ decreases faster than $w_0(p)=\frac{1}{p}$, we will observe a more aggressive prompt reweighting on the low-pass-rate prompts over MaxRL.

For example, consider a random variable $Z$ with a truncated exponential distribution in $[0, 1]$:
\begin{align*}
    f_Z(p) = \frac{\lambda e^{-\lambda p}}{1 - e^{-\lambda }}, \ F_Z(p)=\frac{1 - e^{-\lambda p}}{1 - e^{-\lambda }}, \ p\in[0,1]. 
\end{align*}
If $w_F(p)$ follows the truncated exponential distribution above, we have
\begin{align*}
    w_F(p) = \frac{\lambda e^{-\lambda p}}{1 - e^{-\lambda p}} = \frac{\lambda}{ e^{\lambda p} - 1}, r(p) = \frac{\lambda p}{ e^{\lambda p} - 1} \Rightarrow r^\prime(p) = \lambda \frac{e^{\lambda p} - 1 - \lambda p e^{\lambda p}}{(e^{\lambda p} - 1)^2} < 0,
\end{align*}
as we can easily prove that $g(z)=e^z - 1 - z e^z$ with $z=\lambda p > 0$ is a decreasing function with a negative derivative and $g(z)<g(0)=0$. Another example is a symmetric random variable $Y$ relative to $Z$ in $[0, 1]$, i.e., $Y = 1-Z$. Therefore, we have
\begin{align*}
    f_Y(p) = \frac{\lambda e^{\lambda p}}{e^{\lambda }-1}, \ F_Y(p)=\frac{ e^{\lambda p}-1}{e^{\lambda }-1}, \ p\in[0,1]. 
\end{align*}
Similarly, this distribution implies that
\begin{align*}
    w_F(p) = \frac{\lambda }{1 - e^{-\lambda p}}, r(p) = \frac{p \lambda }{1 - e^{-\lambda p}} \Rightarrow r^\prime(p) = \lambda \frac{1 - e^{-\lambda p}  - \lambda p e^{-\lambda p}}{(1 - e^{-\lambda p} )^2} > 0,
\end{align*}
which indicates that this strategy is more conservative than MaxRL with a slower decreasing speed of $W_F(p)$ than $\frac{1}{p}$. Figure~\ref{fig:app_weightcomparision} showcases the density function, CDF, and weight functions in terms of the pass rate $p$ of MaxRL and two considered examples, with different aggressiveness of weights in the gradient update. 

Empirically, we compare the dynamics of normalized weights among GRPO, MaxRL, and CurveRL in \Cref{fig:dynamics_weights}. It suggests that CurveRL behaves more aggressively and puts larger weights on low-pass-rate prompts than MaxRL with a faster decreasing speed when we increase $p$. We argue that the current model capability is still less sufficient for the dataset POLARIS-53K, such that a more risk-seeking utility function is preferable, which is guided in principle by the relationship between data difficulty and the model's capability. This result demonstrates that our method is data-driven with an adaptive weight function relative to the static ones in GRPO and MaxRL. Conversely, a conservative weight should be observed when we deploy a larger model on an easier dataset. This is an interesting investigation for empirical demonstration, which we leave as future work.

\begin{figure}[t]
    \centering
    \begin{subfigure}{0.51\linewidth}
        \centering
        \includegraphics[width=\linewidth]{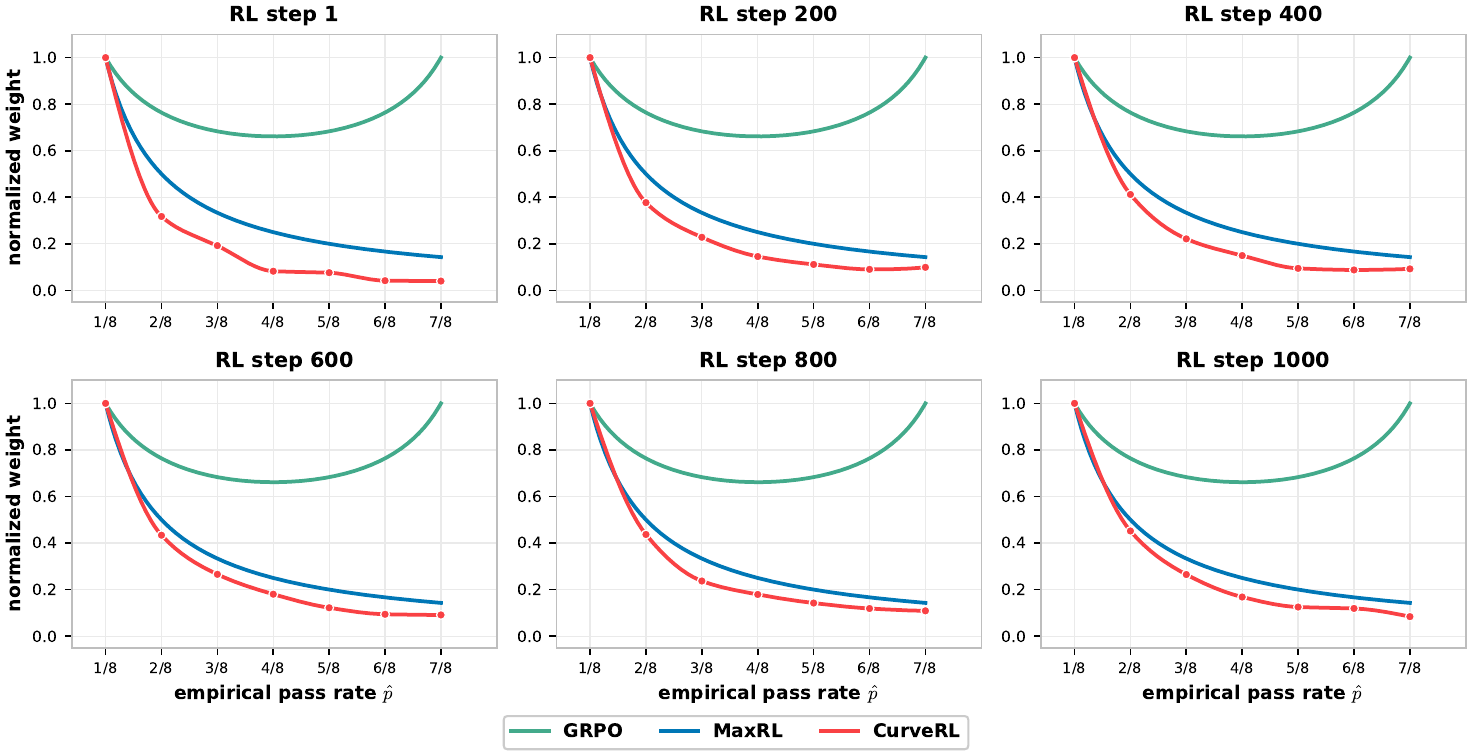}
        \caption{Normalized Weight Dynamics on Six Steps.}
        % \label{fig:empirical-vs-maxrl-weight}
    \end{subfigure}
    \hfill
    \begin{subfigure}{0.46\linewidth}
        \centering
        \includegraphics[width=\linewidth]{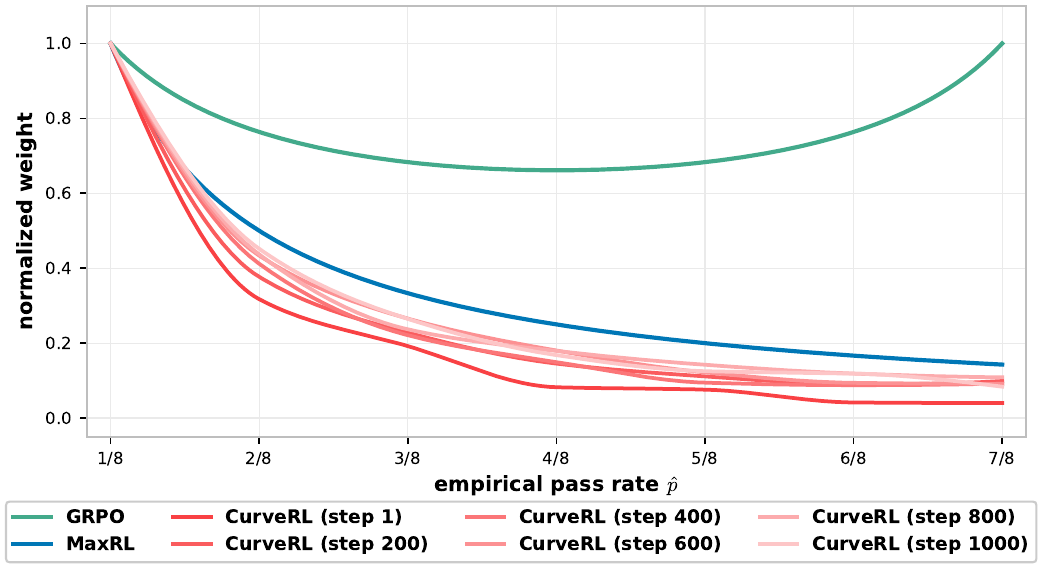}
        \caption{Combined Normalized Weight Dynamics.}
        \label{fig:empirical-weights-combined}
    \end{subfigure}
    \caption{Dynamics of normalized weights among GRPO, MaxRL, and CurveRL on Qwen3-4B-Base.}
    \label{fig:dynamics_weights}
\end{figure}

\paragraph{For the General $\psi$.} Our sufficient condition can be extended to the general risk function $\psi$, where the condition quantity is
\begin{align*}
    R_\psi(p):=\frac{w_F(p)}{w_0(p)}=\frac{\psi^{\prime}\left(F_{\mathrm{ref}}(p)\right) f_{\mathrm{ref}}(p)}{\psi^{\prime}(p)}.
\end{align*}
If $R_{\psi}^\prime(p)<0$, our distribution-aware weight function $w_F(p)$ is more aggressive over the low-pass-rate prompts than the pointwise counterpart in the learning dynamics.

\paragraph{Relationship between Model Capacity and Data Difficulty.} A key observation is that our distribution-aware weight is data-driven and depends on both the model $\pi_\theta$ and the dataset. The mixmatch between the model capability and data difficulty fundamentally determines the dynamics of the distribution-aware weights, including whether the weight is more aggressive than the pointwise version or not. In particular, if a simpler model is employed on a hard dataset, more pass rates of prompts will be more concentrated around 0, leading to more aggressive weight on the low-pass-rate prompts. On the other hand, if a large model is employed on an easy dataset, the weight distribution is more spread and tends to be more uniform. Our distribution-aware weight likely behaves similarly as the pointwise counterpart, e.g., MaxRL, or even more conservatively on low-pass-rate prompts. We leave more empirical demonstration as future work.

\subsection{Induced Wasserstein-type Geometry in Our Method}\label{app:discussion_geometry}

The transform $F_{\text{ref}}$ provides a quantile-based difficulty coordinate: it measures the relative position of the pass rate $p_\theta(x)$ under a fixed reference distribution. This representation depends on the ordering structure of pass rates rather than their raw scale. In this sense, it is closer in spirit to one-dimensional Wasserstein-type quantile representations than to pointwise density-ratio-based objectives such as Kullback–Leibler~(KL) divergence.

\paragraph{Kullback–Leibler~(KL) divergence with pointwise density geometry.} When we compare the distances of two one-dimensional distributions $\mu$ and $\nu$, if they are absolutely continuous with respect to Lebesgue measure, the densities $p_\mu$ and $p_\nu$ exist. Consequently, the Kullback–Leibler~(KL) divergence is defined by using the pointwise probability ratio:
\begin{align*}
    D_{\text{KL}}(\mu, \nu) = \int_{\mathbb{R}} p_{\mu}(x) \log \frac{p_\mu(x)}{p_\nu(x)} dx,
\end{align*}
which induces a pointwise geometry in the density space.

\paragraph{Optimal Transport~(OT)-based Distance with expressive geometry.}
As a typical instance of optimal transport, $p$-Wasserstein distance on one-dimensional random variables is defined by
\begin{align*}
    W_p^p(\mu, \nu)=\inf _{\gamma \in \Pi(\mu, \nu)} \int\|x-y\|^p d \gamma(x, y),
\end{align*}
where $\gamma$ is the coupling or the transport plan, and $\Pi(\mu, \nu)$ is the joint distribution space with the marginal distributions as $\mu$ and $\nu$. By definition, Wasserstein distance explicitly accounts for the underlying geometry of the data space, as opposed to the KL divergence. In the one-dimensional space, Wasserstein distance is equivalent to the \textbf{quantile matching}:
\begin{align*}
    W_p^p(\mu, \nu)=\int_0^1\left|F_\mu^{-1}(u)-F_\nu^{-1}(u)\right|^p d u,
\end{align*}
which evaluates the distance in the quantile space with respect to the quantile function $F^{-1}$. When $p=1$, 1-Wasserstein distance has an equivalent form based on the CDF:
\begin{align*}
    W_1(\mu, \nu)=\int_{\mathbb{R}}\left|F_\mu(t)-F_\nu(t)\right| d t,
\end{align*}
which induces a transport or ordering geometry.

In summary, our method operates $F_{\text{ref}}(p_\theta(x))$ in the distribution-aware utility functional, a surrogate of the one from the perspective of distribution transport~(moving the mass from $p_\theta(x)$ to $p^*(x)$) as introduced in Section~\ref{sec:optimality_utility}. Since optimal transport in the one-dimensional space reduces to quantile matching and CDF matching, our method effectively operates in the quantile space~(equivalent to the CDF space), inducing a Wasserstein-type geometry. 

\subsection{Integrating both Pointwise and Distribution-aware Utility Functionals}\label{app:discussion_integration}

\paragraph{Strategy 1: Single Function Transformation.} We integrate both $p_\theta(x)$ and $F_{\mathrm{ref}}\left(p_\theta(x)\right)$ into a single function transformation $\phi$:
\begin{align*}
    \mathcal{U}_\theta=\mathbb{E}_{x\sim d_0}\left[\phi\left(p_\theta(x), F_{\mathrm{ref}}\left(p_\theta(x)\right)\right)\right].
\end{align*}
We choose $\phi(a, b)=\log(a^{1-\lambda} b^\lambda)$ for computational convenience, which leads to
\begin{align*}
    \mathcal{U}_\theta=\mathbb{E}_{x\sim d_0}\left[\log \left(p_\theta(x)^{1-\lambda} F_{\text{ref}}\left(p_\theta(x)\right)^\lambda\right)\right] = (1-\lambda) \mathbb{E}_{x\sim d_0}\left[\log p_\theta(x)\right] + \lambda \mathbb{E}_{x\sim d_0}\left[\log F_{\text{ref}}\left(p_\theta(x)\right)\right].
\end{align*}
Notably, this utility functional is naturally meaningful as it is a monotonically increasing function of $p_\theta(x)$. Therefore, by applying the functional derivative in Definition~\ref{def:weight}, the optimal weight function and gradient update rule is:
\begin{align*}
    \nabla_\theta J_{\text{integrated}}(\theta) &  = \mathbb{E}_{x\sim d_0}\left[\left((1-\lambda) \frac{1}{p_\theta} + \lambda \frac{f_{\text{ref}}\left(p_\theta(x)\right)}{F_{\text{ref}}\left(p_\theta(x)\right)}\right)\nabla_\theta p_\theta(x)\right].
\end{align*}

\paragraph{Strategy 2: Multiplication.} Alternatively, we can directly multiply the two function transformations $\psi$ and $g$ under a monotonic constraint to ensure a meaningful utility function over $p_\theta(x)$:
\begin{align*}
    & \mathcal{U}_\theta=\mathbb{E}_{x \sim d_0}\left[\psi\left(F_{\text{ref}}\left(p_\theta(x)\right)\right) g\left(p_\theta(x)\right)\right], \\
    &\text{s.t.} \ \frac{\delta U_\theta}{\delta p_\theta}(x) = \psi\left(F_{\text {ref }}(p_\theta(x))\right) g^{\prime}(p_\theta(x))+\psi^{\prime}\left(F_{\text {ref }}(p_\theta(x))\right) f_{\text {ref }}(p_\theta(x)) g(p_\theta(x)) > 0.
\end{align*}
where the monotonic constraint is derived by ensuring a positive first-order derivative in terms of $p_\theta(x)$. When we employ $\psi(\mu)=-\log(\mu)$ and $g(\nu)=\log(\nu)$, it guarantees $\frac{\delta U_\theta}{\delta p_\theta}(x) > 0$. Consequently, the integrated gradient update rule is:
\begin{align}
      \nabla_\theta J_{\text{integrated}}(\theta) &  = \mathbb{E}_{x\sim d_0}\left[ - \left( \frac{\log F_{\text {ref }}\left(p_\theta(x)\right)}{p_\theta(x)}+\frac{f_{\text {ref }}\left(p_\theta(x)\right) \log p_\theta(x)}{F_{\text {ref }}\left(p_\theta(x)\right)} \right)\nabla_\theta p_\theta(x) \right].
\end{align}

\paragraph{Remark.} Although the above integrated formulations provide a natural way to combine local sensitivity and distributional information, we do not observe empirical improvements in practice. This suggests that simply combining pointwise and distribution-aware terms does not necessarily lead to more effective learning signals due to the \textit{geometry mismatch or gradient conflicts}. This observation demonstrates that the benefit of distribution-aware objectives does not arise from a direct additive or multiplicative fusion with pointwise terms, but rather from explicitly modeling the geometry of the pass-rate distribution. Therefore, it further helps to posit our contribution: purely distribution-aware utility functionals with clearer geometric interpretation are not considered as a regularization or supplement to pointwise ones, but correspond to a fundamentally different objective class that derives a new separate weight function family.

\end{document}